%% file: sIPM_LFR_cameraready_final.tex
\documentclass[nohyperref]{article}

\usepackage{microtype}
\usepackage{graphicx}
\usepackage{subfigure}
\usepackage{booktabs} 

\usepackage{amsmath, amsfonts, amsthm}
\usepackage{amssymb}
\usepackage{multirow}
\usepackage{bbm}
\usepackage{algorithm}
\usepackage{algorithmic}

\input{math_commands.tex}

\usepackage{hyperref}

\usepackage[accepted]{icml2022}

\usepackage{amsmath}
\usepackage{amssymb}
\usepackage{mathtools}
\usepackage{amsthm}

\usepackage[capitalize,noabbrev]{cleveref}

\theoremstyle{plain}

\newtheorem{theorem}{Theorem}[section]
\newtheorem{proposition}[theorem]{Proposition}
\newtheorem{lemma}[theorem]{Lemma}

\theoremstyle{definition}

\theoremstyle{remark}

\def\cV{{\cal V}}

\def\cH{{\cal H}}
\def\cF{{\cal F}}

\def\cG{{\cal G}}
\def\cX{{\cal X}}

\def\cZ{{\cal Z}}

\def\cC{{\cal C}}

\def\qed{\space$\square$ \par \vspace{.15in}}

\def\hat{\widehat}

\newcommand{\mbP}{\mathbb{P}}

\newcommand{\br}{{\bf r}}

\newcommand{\bZ}{{\bf Z}}

\newcommand{\bz}{{\bf z}}

\newcommand{\bX}{{\bf X}}
\newcommand{\bx}{{\bf x}}

\newcommand{\bV}{{\bf V}}

\newcommand{\bU}{{\bf U}}

\newcommand{\balpha}{\mbox{\boldmath{$\alpha$}}}

\newcommand{\bc}{\begin{center}}
\newcommand{\ec}{\end{center}}
\newcommand{\be}{\begin{equation}}
\newcommand{\ee}{\end{equation}}
\newcommand{\ba}{\begin{array}}
\newcommand{\ea}{\end{array}}
\newcommand{\bean}{\begin{eqnarray*}}
\newcommand{\eean}{\end{eqnarray*}}
\newcommand{\bea}{\begin{eqnarray}}
\newcommand{\eea}{\end{eqnarray}}
\newcommand{\ben}{\begin{enumerate}}
\newcommand{\een}{\end{enumerate}}
\newcommand{\bed}{\begin{itemize}}
\newcommand{\eed}{\end{itemize}}

\usepackage[textsize=tiny]{todonotes}

\icmltitlerunning{Learning fair representation with a parametric integral probability metric}

\begin{document}

\twocolumn[
\icmltitle{Learning fair representation with a parametric integral probability metric}




\begin{icmlauthorlist}
\icmlauthor{Dongha Kim}{swu,dsc}
\icmlauthor{Kunwoong Kim}{snu}
\icmlauthor{Insung Kong}{snu}
\icmlauthor{Ilsang Ohn}{iu}
\icmlauthor{Yongdai Kim}{snu}
\end{icmlauthorlist}

\icmlaffiliation{snu}{Department of Statistics, Seoul National University}
\icmlaffiliation{swu}{Department of Statistics, Sungshin Women's University}
\icmlaffiliation{dsc}{Data Science Center, Sungshin Women's University}
\icmlaffiliation{iu}{Department of Statistics, Inha University}

\icmlcorrespondingauthor{Yongdai Kim}{ydkim0903@gmail.com}

\icmlkeywords{Machine Learning, ICML}

\vskip 0.3in
]



\printAffiliationsAndNotice{}  

\begin{abstract}

As they have a vital effect on social decision-making, AI algorithms should be not only accurate but also fair. Among various algorithms for fairness AI, learning fair representation (LFR), whose goal is to find a fair representation with respect to sensitive variables such as gender and race, has received much attention. For LFR, the adversarial training scheme is popularly employed as is done in the generative adversarial network type algorithms. The choice of a discriminator, however, is done heuristically without justification. In this paper, we propose a new adversarial training scheme for LFR, where the integral probability metric (IPM) with a specific parametric family of discriminators is used. 
The most notable result of the proposed LFR algorithm is its theoretical guarantee about the fairness of the final prediction model, which has not been considered yet. That is, we derive theoretical relations between the fairness of representation and the fairness of the prediction model built on the top of the representation (i.e., using the representation as the input). 
Moreover, by numerical experiments, we show that our proposed LFR algorithm is computationally lighter and more stable, and the final prediction model is competitive or superior to other LFR algorithms using more complex discriminators.

\end{abstract}

\section{Introduction}
\label{intro}
Artificial intelligence (AI) has accomplished tremendous success in various real-world domains.   
The key of success of AI is ``learning from data''.
However, in many cases, data include
historical bias against certain socially sensitive groups such as
gender, race, religion, etc \cite{feldman2015certifying, angwin2016machine, kleinberg2018algorithmic, mehrabi2019survey},
and trained AI models from such biased data
could also impose bias or unfairness against sensitive groups.
As AI has a wide range of influences on human social life, issues of transparency and ethics of AI are emerging.  Therefore, designing an AI algorithm which is accurate and fair simultaneously has become a crucial research topic \cite{calders2009building, feldman2015certifying, barocas2016big, hardt2016equality, zafar2017fairness, donini2018empirical, pmlr-v80-agarwal18a, quadrianto2019discovering}.

Among various researches related to fair AI, learning fair representation (LFR)
has received much attention recently \cite{pmlr-v28-zemel13, 8622525, Quadrianto_2019_CVPR, ruoss2020learning, pmlr-v130-gitiaux21a, zeng2021fair}.
Fair representation typically means a feature vector obtained by transforming the data
such that the distributions of the feature vector for each sensitive group are similar.
Once the fair representation is learned, any prediction models constructed on the top of the fair 
representation (i.e. using the representation as an input vector) are expected to be fair
\cite{pmlr-v28-zemel13, Madras2018LearningAF}. 


A popular approach for LFR is to use the adversarial training scheme \cite{9aa5ba8a091248d597ff7cf0173da151,Madras2018LearningAF}. 
As is done in the generative adversarial network (GAN, \citet{NIPS2014_5ca3e9b1}), the algorithm seeks a representation that fools the discriminator the best
that tries to predict which sensitive group a given representation belongs.
Different algorithms to learn the discriminator result in different algorithms for LFR. 

Despite their considerable success, there are still 
theoretical and practical limitations in the existing learning algorithms for fair representation based on the adversarial training scheme.
First of all, it is not clear how the level of fairness of the representation
affects the level of fairness of the final prediction model (built on the top of the representation).
This problem is important since the final goal of LFR is to construct fair prediction models.

In this paper, we consider the adversarial training scheme based on the integral probability metric (IPM). The IPM, which includes
the Wasserstein distance \cite{KR:58, Villani2008OptimalTO} as a special case, has been widely used for learning
generative models (e.g. Wasserstein GAN, \citet{10.5555/3305381.3305404}), but has not been used for fair
representation. An advantage of using the IPM is that we can control the level of fairness of the final prediction model by controlling the level of fairness of the representation relatively easily.

The second problem we study, which is the main contribution of this paper, is the choice of the class of discriminators.
Deep neural networks (DNNs) are popularly used for the discriminator \cite{NIPS2014_5ca3e9b1, 10.5555/3305381.3305404, Madras2018LearningAF, pmlr-v97-creager19a, ansari2020characteristic},
but the choice of the architecture (the numbers of layers and nodes at each layer)
is decided rather heuristically without justification. 
In this paper, we propose a specific parametric family of discriminators and provide theoretical guarantees of the fairness of the final prediction models in terms of the fairness of the representation for large classes of prediction models.

By applying the IPM with the proposed parametric family of discriminators, 
we propose a new learning algorithm for fair representation abbreviated by
the sIPM-LFR (sigmoid IPM for Learning Fair Representation).
Along with the theoretical guarantees, the sIPM-LFR has several advantages over existing LFR algorithms. For example, the sIPM-LFR is computationally lighter, more stable, and less prone to bad local minima. Moreover, the final prediction model is competitive or superior in prediction performance to those from other LFR algorithms. 


This paper is organized as follows. 
In Section \ref{relat}, we review related studies about fairness of  AI. 
The sIPM-LFR algorithm is proposed in Section \ref{not_pre}, and
the results of theoretical studies are presented in Section \ref{theory}.
Numerical studies are conducted in Section \ref{exps} and concluding remarks follow in Section \ref{conclusion}. 

The main contributions of this work are summarized as follows.
\begin{itemize}
    \item We propose a simple but powerful fair representation learning method by developing a new adversarial training scheme based on a parametric IPM.
    \item We give theoretical guarantees about fairness of the final prediction model in terms of fairness of the representation.
    \item We empirically show that our algorithm is competitive or even superior to other existing LFR algorithms.
\end{itemize}

\section{Related works}
\label{relat}

\paragraph{Algorithmic fairness}

Generally, various concepts of fair prediction models can be summarized into three categories. 
The first category is \textit{group fairness} which requires that certain statistics of the prediction model at each sensitive group are similar \cite{calders2009building, barocas2016big}. 

The second notion of fair prediction models is  \textit{individual fairness}, which aims at treating similar inputs similarly \cite{dwork} regardless of sensitive groups.
Various practical algorithms and their theoretical properties have been proposed and studied by \citet{pacf, average, face, sensr}.

The third concept of fair prediction models is \textit{counterfactual fairness} \cite{NIPS2017_a486cd07}, which can be considered as a compromise between group fairness and individual fairness.
Simply speaking, counterfactual fairness requires that similar individuals only from different sensitive groups should have similar prediction values. 
The notion of counterfactual is used to define similar individuals from different sensitive groups \cite{ijcai2019-199, Chiappa_2019, 10.1145/3306618.3317950}.

\paragraph{Learning fair representations}

LFR has a different strategy than
the fair AI algorithms mentioned in the previous subsection.
Instead of constructing fair prediction models directly, LFR 
first constructs a fair representation such that the distributions of the representation for each
sensitive group are similar. Then, LFR learns a prediction model 
on the top of the representation (i.e. using the fair representation as an input).
LFR has been initially considered by \citet{pmlr-v28-zemel13},
and many advanced algorithms have been developed \cite{8622525, pmlr-v97-creager19a, Quadrianto_2019_CVPR, ruoss2020learning, pmlr-v130-gitiaux21a, zeng2021fair} afterward.

One of the most pivotal learning frameworks of LFR is the adversarial training scheme \cite{9aa5ba8a091248d597ff7cf0173da151,Madras2018LearningAF}.
Those algorithms try to fool a given discriminator similar to that of GAN does \cite{NIPS2014_5ca3e9b1}. The aim of this paper is to propose a new adversarial training scheme for LFR which is computationally easier and has desirable theoretical guarantees.

\section{Learning fair representation by use of a parametric IPM}
\label{not_pre}

In this section, we propose a new learning algorithm for fair representation. 
In particular, we develop a parametric IPM to measure the fairness of a given representation mapping. 
We first review the population version of the existing learning algorithms for fair representation and explain problems when we modify the population version to the sample version and propose a parametric IPM to resolve the problems.

\subsection{Notations and Preliminaries} 

\paragraph{Notations} 
Let  $\mathbf{X} \in \mathcal{X} \subset \mathbb{R}^{d}$, $S \in \{0, 1\}$, 
and $Y \in \{0, 1\}$ be the non-sensitive random input vector, (binary) sensitive random input variable and (binary) output variable whose joint distribution
is $\mathbb{P}.$
 Also let $\mathbf{Z} := h(\mathbf{X},S)$ be the representation of an input vector
$(\bm{X},S)$ obtained by an encoding function $h:\cX\times\{0,1\}\to\cZ \subset \mathbb{R}^{m}$. 
Note that we allow the encoding function depending on both non-sensitive and sensitive inputs as \citet{Madras2018LearningAF} did. 
Let $f:\cZ\to \mathbb{R}$ and $f_D:\cZ\to\cX\times\{0,1\}$ be a prediction model and a decoding function, respectively. 
For technical simplicity, we assume that $\cZ$ is bounded and
$\sup_{\bz\in \cZ} |f(\bz)| \le F$ for some constant $F>0.$

\paragraph{Fairness for DP} 

Fair representation is closely related to demographic parity (DP) which is a concept for group fairness.
In fact, we will see later that the prediction model $f\circ h$ can be fair in view of DP 
when the representation $\mathbf{Z}$ is fair in a certain sense.
Here, we briefly review the notion of fairness for DP.

Let $\phi$ be a function from $\mathbb{R}$ to $\mathbb{R}.$
For a given prediction model $g:\cX\times \{0,1\} \rightarrow \mathbb{R},$
we say that the level of $\phi$-fairness of $g$ is $\epsilon$ if 
$DP_\phi(g)< \epsilon,$ where
\begin{equation}
\label{eq:DP}
DP_\phi(g)= |\E(\phi\circ g(\bm{X},S)|S=0)-\E(\phi\circ g(\bm{X},S)|S=1)|.
\end{equation}
Various definitions of DP-fairness are special cases
of the $\phi$-fairness. The original DP-fairness uses $\phi(w)=\mathbb{I}(w\ge 0)$ \cite{calders2009building, barocas2016big}, and $\phi(w)=(w+1)_{+}$
is popularly used as a convex surrogate of $\mathbb{I}(w\ge 0)$ \cite{10.1145/3308558.3313723, pmlr-v119-lohaus20a}.
When $\phi(w)=w,$ the corresponding fairness measure becomes the mean DP (MDP, \citet{Madras2018LearningAF, chuang2021fair}). 

\subsection{Description of LFR algorithms}

The goal of LFR is to find an encoding function $h$ such that 
\be
\label{eq:ed1}
 \mathbb{P}\left\{ h(\bm{X},S) \in \cdot|S=0 \right\} \approx  \mathbb{P}\left\{ h(\bm{X},S) \in \cdot|S=1 \right\}.
\ee
Once we have the encoding function, we construct a prediction model on
the representation space $\cZ.$ That is, the final prediction model $g$ is given as
$g(\bm{x},s)= f \circ h(\bm{x},s),$ where $f$ is a prediction model from $\cZ$ to $\mathbb{R}.$ 
Due to (\ref{eq:ed1}),
we expect that 
$$
 \mathbb{P}\left\{ g(\bm{X},S) \in \cdot|S=0 \right\} \approx  \mathbb{P}\left\{ g(\bm{X},S) \in \cdot|S=1 \right\}.
$$
and thus the prediction model is expected to be DP-fair.

The basic algorithm of LFR consists of the following two steps.
The first step is to choose a deviance measure $d$ between two distributions and a class $\cH$ of encoding functions
and the second step is to find an encoding function $h$ which minimizes $d(\mathbb{P}_0^h,\mathbb{P}_1^h),$
where $\mathbb{P}_s^h$ is the conditional distribution of $h(\bm{X},S)$ given $S=s$ for $s\in \{0,1\}.$

In turn, to define a deviance measure, the adversarial training scheme is popularly employed. 
For a given class of discriminators $\cV$ and a given classification loss $l,$ one possible deviance measure is defined as
$ d(\mathbb{P}_0^h,\mathbb{P}_1^h) =\sup_{v\in \cV} \E\left\{l(S, v\circ h(\bm{X},S)\right\}.$
Various classification losses have been used for learning fair representation:
\citet{9aa5ba8a091248d597ff7cf0173da151} uses the cross-entropy loss and \citet{Madras2018LearningAF} uses the $L_{1}$ loss.

The minimizer of $d(\mbP_{0}^h,\mbP_{1}^h),$ however, is not unique in most cases. 
For example, if there exists $h$ 
such that $d(\mbP_{0}^h,\mbP_{1}^h)=0,$ then any encoding function given as $\xi\circ h$ for any $\xi: \cZ \rightarrow \cZ$ also has
the zero deviance. 
Also, an encoder derived as such might not provide helpful information (e.g., $h(\cdot)=0$). 

There are two ways to resolve these problems in the adversarial training scheme for LFR - supervised and unsupervised methods. 
For the supervised adversarial training scheme, we choose
a set $\cF$ of prediction models on $\cZ$ and then learn $h$ as well as $f$ by minimizing 
\begin{equation}
\label{eq:obj-sup}
L(f\circ h)+\lambda d(\mbP_{0}^h, \mbP_{1}^h)
\end{equation}
in $f\in \cF$ and $h\in \cH,$ where $L$ is a certain classification risk for $Y$ such as the cross-entropy and $\lambda>0$ is a regularization parameter. 

For the unsupervised adversarial training scheme,
we first choose a set $\cF_D$ of decoding functions from $\cZ$ to $\cX\times \{0,1\},$
then we learn the encoding function by minimizing 
\begin{equation}
\label{eq:obj-unsup}
L_{recon}(f_D\circ h)+ \lambda d(\mbP_{0}^h, \mbP_{1}^h),
\end{equation}
where $L_{recon}$ is a reconstruction error.
When the learning procedure of $h$ finishes, the extracted fair representation are used to solve various downstream classification tasks. 
That is, we do not use the label information $Y$ when we learn $h,$ which is an advantage of the unsupervised adversarial training scheme.

When we do not know the population distribution $\mbP$ but we have data,
a standard method of LFR is to replace $\mbP$
by its empirical counterpart 
$\mbP_n(\cdot)=\sum_{i=1}^n \delta_{(\bx_i,y_i,s_i)}(\cdot)/n,$
the empirical distribution, where $\delta_a$ is the Dirac-delta function and $\{(\bx_i,y_i,s_i)\}_{i=1}^n$ is a given training dataset.

Regarding optimizing the formulas (\ref{eq:obj-sup}) and (\ref{eq:obj-unsup}) in practice, obtaining
the value of $d(\mathbb{P}_{0}^h, \mathbb{P}_{1}^h)$
 is time-consuming since we have to find a discriminator maximizing the classification loss of $S$
 (i.e. $\sup_{v\in \cV} \E\left\{l(S, v\circ h(\bm{X},S)\right\}$).
To reduce this computational burden, at each update, we apply a gradient ascent algorithm to update the parameters
in the discriminator few times, e.g. five times, as is done by \citet{NIPS2014_5ca3e9b1}. 

{The aim of this paper is to propose a novel measure for 
$d(\mathbb{P}_{0}^h, \mathbb{P}_{1}^h)$ used in (\ref{eq:obj-sup}) and (\ref{eq:obj-unsup}), which we will describe in the subsequent sections. 
For details of the corresponding learning algorithm, see Section \ref{appendix:alg}.
}

\subsection{Learning fair representation with IPM}

In this paper, we consider the integral probability metric (IPM) as 
the deviance measure for LFR. For a given class $\cV$ of discriminators from $\cZ$ to $\mathbb{R},$ the IPM $d_{\cV}(\mbP_0,\mbP_1)$ 
for given two probability measures $\mbP_0$ and $\mbP_1$ is defined as
$$d_{\cV}(\mbP_0,\mbP_1)=\sup_{v\in \cV} \left| \int v(\bm{z}) (d\mbP_0(\bm{z})-d\mbP_1(\bm{z}))\right|.$$
When $\cV$ includes all Lipschitz functions\footnote{A given function $v$ on $\cZ$ is a Lipschitz function with the Lipschitz constant $L$ if
$|v(\bz_1)-v(\bz_2)| \le L \|\bz_1-\bz_2\|$ for all $\bz_1,\bz_2\in \cZ,$
 where $\|\cdot\|$ is certain norm defined on $\cZ.$}, then the IPM becomes the well known Wasserstein norm \cite{KR:58}. Even if it is popularly used in various applications of AI including the  generative model learning, the IPM has not been studied
deeply for LFR.

An obvious advantage of the IPM compared to the other deviance measures is that the level of the IPM is directly related to the level of DP-fairness of the final prediction model. That is, suppose that
a given encoding function $h$ satisfies $d_{\cV}(\mbP_{0}^h,\mbP_{1}^h)<\epsilon,$ then any prediction model given as $f \circ h$ automatically satisfies 
the level of $\phi$-fairness less $\epsilon,$ as long as $\phi\circ  f$ belongs to $\cV.$
For example, suppose that $\cV$ is the set of Lipschitz continuous functions.
If $\phi$ is a Lipschitz function with the Lipschitz constant less than or equal to 1, then $\phi\circ f$ belongs to $\cV$ whenever $f\in \cV.$
Examples of $\phi$ with the bounded Lipschitz constant are $\phi(w)=w$ and $\phi(w)=(w+1)_{+}.$

\subsection{The sigmoid IPM: A parametric IPM for fair representation}

We need to set in advance the function spaces for $\cH, \cF,$ and $\cF_D$ as well as $\cV$ to make the minimization of the regularized empirical risk in (\ref{eq:obj-sup}) or
(\ref{eq:obj-unsup}) be possible.
There are many well known and popularly used models for $\cH$ (e.g. DNN and 
ConvNet), $\cF$ (e.g. linear, DNN, and Kernel machine 
\cite{10.1023/A:1022627411411}),
and $\cF_D$ (e.g. DNN and DeConvNet \cite{noh2015learning}). In contrast, the choice of $\cV$ is typically done heuristically. 
DNNs are popularly used for $\cV$ \cite{10.5555/3305381.3305404},
but the choice of the architecture (the numbers of layers and nodes at each layer)
is decided without justification. In this subsection, we focus on the choice of
$\cV$ and propose a specific parametric family with theoretical justifications
in view of DP-fairness.

Suppose that $\cH$ and $\cF$ are given. That is, the final prediction model
is given as $f\circ h,$ where $f\in \cF$ and $h\in \cH.$ Also, the fairness function $\phi$ is given. Our mission is to choose $\cV$ such that
the level of $\phi$-fairness of the final prediction model can be controlled by controlling 
the $d_{\cV}(\mbP_0^h,\mbP_1^h).$ This is an important task for the unsupervised LFR since the label $Y$ is not available when fair representation
is learned.

To be more specific, we derive a non-decreasing function $\rho:\mathbb{R}_+ \rightarrow \mathbb{R}_+$
such that 
$$\sup_{f\in \cF} DP_{\phi}(f\circ h) \le \rho\left\{d_{\cV}(\mbP_0^h,\mbP_1^h)\right\}.$$
That is, we can control the $\phi$-fairness of any $f\circ h$
 by controlling the $d_{\cV}$ of $h.$

A naive choice of $\cV$ would be that $\phi\circ f \in \cV$ for all $f\in \cF,$
in which case $\rho(\epsilon)=\epsilon.$ Such a choice, however,
is not possible for the unsupervised LFR since the prediction model
space $\cF$ is selected after learning the fair representation.
One may choose a very large $\cV$ so that $\phi\circ \cF \subset \cV$
for most classes of $\cF.$ Such a choice, however, would make the computational cost unnecessarily large and increase
the variance of the learned model due to too many parameters in $\cV$ to degrade performance.

We explore an opposite direction: to seek a class of $\cV$ that is small but controls the level of $\phi$-fairness easily.
In this paper, we propose a specific parametric family for $\cV$
and show that the $\phi$-fairness of $f\circ h$ can be controlled
nicely by $d_{\cV}$ of $h$ for fairly large classes of $\cF.$ 

In fact, using the parametric IPM is not new. \citet{ansari2020characteristic} considers $\cV_{char}=\{\exp(i\mathbf{t}^\top \bm{x}): \mathbf{t}\in\mathbb{R}^m\}$
in the GAN algorithm. This class of functions are related to the characteristic function and it is easy to see that $d_{\cV_{char}}(\mbP_0,\mbP_1)=0$
if and only if $\mbP_0(\cdot) \equiv \mbP_1(\cdot).$ However, it is not clear
what happens when $d_{\cV_{char}}(\mbP_0,\mbP_1) < \epsilon.$
That is, not much is known about
which quantities of $\mbP_0$ and $\mbP_1$ are similar.
\citet{mccullagh1994does} noticed that $d_{\cV_{char}}$ would not be a useful metric between probability measures.

The parametric family we propose in this paper is
\begin{align} \label{g_sigma}
    \cV_{sig}=\{ \sigma(\theta^\top \bm{x}+\mu): \theta\in \mathbb{R}^m,
\mu\in \mathbb{R}\},
\end{align}
where $\sigma(z)=(1+\exp(z))^{-1}$ is the sigmoid function. 
It is surprising to see that the IPM with this simple $\cV_{sig}$ can control
the level of $\phi$-fairness of $f\circ h$ for diverse classes of $\cF,$  whose
results are rigorously stated in the following section. 

Before going further, we give a basic property of the IPM with $\cV_{sig},$ whose proof is stated in Appendix \ref{appendix:proofs}.

{\proposition{\label{pro:1}
For two probability measures $\mbP_0$ and $\mbP_1,$
$d_{\cV_{sig}}(\mbP_0,\mbP_1)=0$ if and only if $\mbP_0(\cdot) \equiv \mbP_1(\cdot).$
}}


\section{Theoretical studies of the IPM with $\cV_{sig}$}
\label{theory}


One may concern that the final prediction model $f\circ h$
would not be fair because the class $\cV_{sig}$ of discriminators
is too small. In this section, we show that the IPM with $\cV_{sig}$
can control the level of $\phi$-fairness of $f\circ h$ for quite large classes of $\cF$ even if $\cV_{sig}$ is small.

We start with the DP-fairness of the perfectly fair representation,
which is a direct corollary of Proposition \ref{pro:1}. 
We defer the proofs of all the following theorems to Appendix \ref{appendix:proofs}.

\begin{theorem}
\label{thm1}
If $d_{\cV_{sig}} (\mbP_{0}^h, \mbP_{1}^h) = 0,$ then the $\phi$-fairness of any prediction model
$f\circ h$ is always 0.
\end{theorem}

It is not realizable to get a perfectly fair representation in practice. Instead, we learn an encoding function whose IPM value is close to 0. In the next two subsections,
we quantify how small the level of $\phi$-fairness of $f\circ h$ when the $d_{\cV_{sig}}$ of $h$ is small for various classes $\cF$ of $f.$

For technical simplicity, we only consider $\phi$ being a polynomial function (i.e. $\phi(w)=w^k$).
Note that reasonably smooth functions can be approximated
by linear combinations of low order polynomial functions. Hereafter, we denote $\phi_k(w)=w^k.$

\subsection{DP-fairness when $\cF$ is well approximated
by a shallow neural network} 

There is much literature about classes of functions that are
well approximated by shallow neural networks with the sigmoid activation function \cite{barron1993universal, yukich1995sup}. In this section, we show that the level of DP-fairness of such functions can
be controlled by the level of the sigmoid IPM.

We consider the class $\cF_{a,C}$ of functions 
considered by \citet{barron1993universal, yukich1995sup}: 
$$\cF_{a,C}=\left\{f: 
\int |\tilde{f}(\bm{w})|d\bm{w}\le a,
\int \|\bm{w}\|_1 |\tilde{f}(\bm{w})|d\bm{w}\le C\right\}$$
for positive constants $a$ and $C,$
where $\tilde{f}(\bm{w})=\int e^{-i \bm{w}^\top \bz} f(\bz) d\bz.$

It is known that any function in $\cF_{a,C}$ can be approximated closely by a single-layered shallow neural network with a finite number of hidden nodes \cite{yukich1995sup}. 
Using this proposition, we have the following theorem, whose proof is deferred to Appendix \ref{appendix:proofs}.

\begin{theorem}
\label{thm3_1}
There exists a constant $c_k>0$ such that
\be
 \sup_{f \in \cF_{a,C}}  DP_{\phi_k}(f\circ h) \le {c_k} \left\{d_{\cV_{sig}}(\mbP_0^h,\mbP_1^h)\right\}^{1/3}. \label{eq:sobolev1}
\ee
\end{theorem}

Theorem \ref{thm3_1} implies that we can control the level of $\phi$-fairness of the final prediction model $f\circ h$
only by making the $d_{\cV_{sig}}$ of the encoding function $h$ sufficiently small. The exponent term $1/3$ on the right hand side of (\ref{eq:sobolev1})
suggests that a smaller value of $d_{\cV_{sig}}$ 
of the encoding function $h$ is needed to control
the level of $\phi$-fairness of $f.$
This is a price we pay for using a simpler class of discriminators.

The exponent $1/3$ in the right-hand side of (\ref{eq:sobolev1}) may not be tight. We can improve this exponent by assuming more on 
$\cF_{a,C}.$ The main message of Theorem \ref{thm3_1} is that
the $\phi$-fairness is controlled when shallow neural networks approximate
the final prediction model well.
However, in the subsequent subsection,
we give an interesting example in that
the sigmoid IPM amply controls the $\phi$-fairness for a class of functions which
is not well approximated by shallow neural networks. 



\subsection{DP-fairness for $f$ being infinitely differentiable}

In general, the encoding function is a complicated mapping (e.g. DNNs) from the input space to the representation space and thus
it is reasonable to expect that
the prediction model from the representation space to the output 
is a simple function such as linear models or sufficiently smooth functions (e.g. the reproducing
kernel Hilbert space (RKHS) with a smooth kernel). Otherwise, the final prediction model would be overly complicated.
For such nice prediction models, we can show that the adversarial training scheme with the sigmoid IPM can control
the level of $\phi$-fairness of the final prediction model more tightly. 

Let $\mathcal{F}_{\mathcal{C}^{\infty},B}$ be the set of infinite times differentiable functions given as 
\begin{align*}
	    \mathcal{F}_{\mathcal{C}^{\infty},B} 
	    = \big\{&f : \cZ \to \mathbb{R} \ : \ \forall \mathbf{r} \in \mathbb{N}_0^m,\\
	    & ||D^{\mathbf{r}} f||_{\infty} \leq \sqrt{{\mathbf{r}}!} B^{|\bm{r}|_1} 
	    \big\}
\end{align*}
for some constant $B>0,$ where $|\bm{r}|_1=\sum_{j=1}^m r_j$ {and $D$ is the derivative operator, that is, for a vector $\br=(r_1,\cdots,r_m)$, $D^{\br}f:=\frac{\partial^{|\br|_1}f}{\partial z_1^{r_1}\cdots\partial z_1^{r_1}}.$}
The specific bound $\sqrt{{\mathbf{r}}!} B^{|\bm{r}|_1} $ for the sup norm of the derivatives
is used for $\mathcal{F}_{\mathcal{C}^{\infty}, B}$ to include some RKHS with smooth kernels (e.g. radial basis function (RBF) kernel).
The following theorem proves that the level of $\phi$-fairness has the same order of the sigmoid IPM for any $f$ in $\mathcal{F}_{\mathcal{C}^{\infty},B}.$

\begin{theorem}
\label{thm3}
There exists a constant $c_k>0$ such that 
\be
 \sup_{f \in \mathcal{F}_{\mathcal{C}^{\infty},B}} DP_{\phi_k}(f\circ h) \le   c_k d_{\cV_{sig}}(\mathbb{P}_0^h, \mathbb{P}_1^h).\label{eq:inf-d}
\ee
\end{theorem}

Theorem \ref{thm3} indicates
that controlling the sigmoid IPM value of the representation $h$ is equivalent to controlling the $\phi$-fairness (up to a constant)
of $f\circ h$ whenever $f\in \mathcal{F}_{\mathcal{C}^{\infty},B}.$
This result justifies the sufficiency of the sigmoid IPM for LFR.

Note that the function class $\cF_{\cC^{\infty},B}$ is large enough to include certain function spaces 
popularly used as the class of prediction models in modern machine learning algorithms.
The RKHS with the RBF kernel is such an example, which is stated in the following proposition.

\begin{proposition}
\label{prop:rbf}
    Let $k_{\gamma}: \mathbb{R}^m \times \mathbb{R}^m \to \mathbb{R}$ be the RBF kernel with the width $\gamma$ defined as
    \begin{align*}
        k_{\gamma}(\bz,\bz^{\prime}) = \exp\left(-\frac{\|\bz - \bz^{\prime}\|_2^2}{\gamma^2}\right)
    \end{align*}
    and $(\mathcal{H}_{\gamma}(\cZ), ||\cdot||_{\mathcal{H}_{\gamma}(\cZ)})$ be the RKHS corresponding to $k_{\gamma}.$
     Define $\cF_{k_\gamma,B}=\{f\in \mathcal{H}_{\gamma}(\cZ): ||f||_{\mathcal{H}_{\gamma}(\cZ)} \le B\}$ for $B>0.$
Then, there exists a $B'>0$ such that $\cF_{k_\gamma,B} \subset \cF_{\cC^{\infty},B'}.$
\end{proposition}




\begin{figure}[t]
\vskip 0.2in
\begin{center}
\centerline{
    \includegraphics[width=0.24\textwidth]{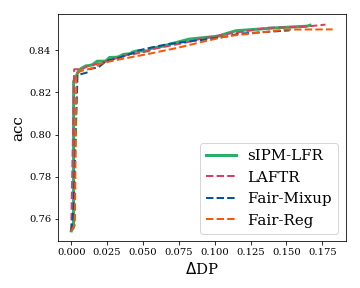}
    \includegraphics[width=0.24\textwidth]{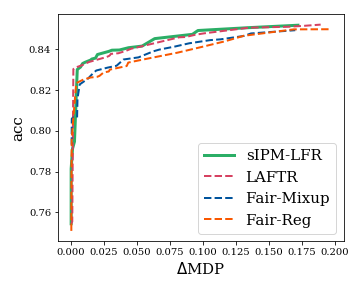}
}
\centerline{
    \includegraphics[width=0.24\textwidth]{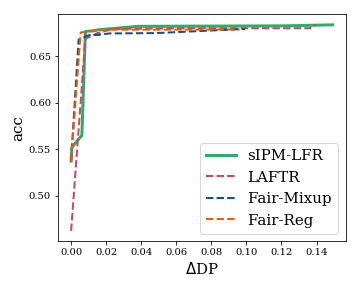}
    \includegraphics[width=0.24\textwidth]{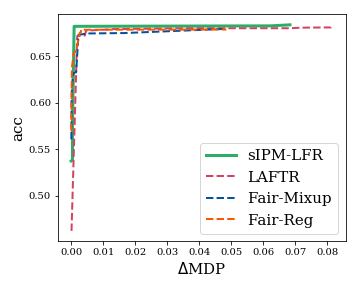}
}
\centerline{
    \includegraphics[width=0.24\textwidth]{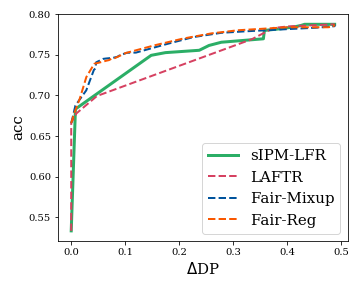}
    \includegraphics[width=0.24\textwidth]{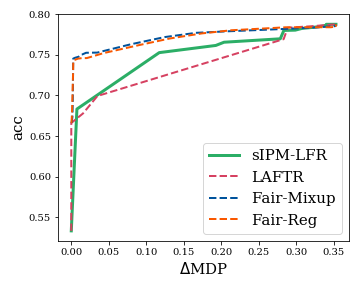}
}
\caption{Supervised LFR: Pareto-front lines between the levels of DP-fairness and \texttt{acc} on the test data of (top) \textit{Adult}, (middle) \textit{COMPAS}, and (bottom) \textit{Health}. For the fairness measure,  (left) $\Delta\texttt{DP}$ and (right) $\Delta\texttt{MDP}$ are considered.}
\label{fig:sup}
\end{center}
\vskip -0.2in
\end{figure}

\begin{figure*}[ht]
\vskip 0.2in
\begin{center}
\centerline{
    \includegraphics[width=0.19\textwidth]{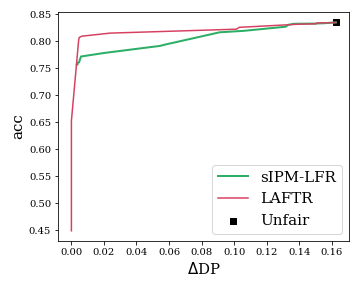}
    \includegraphics[width=0.19\textwidth]{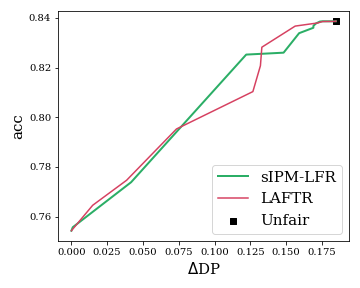}
    \includegraphics[width=0.19\textwidth]{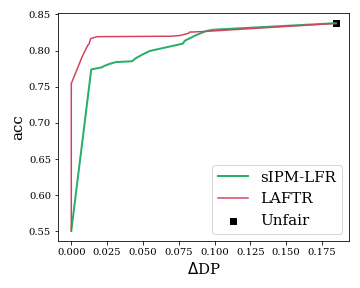}
    \includegraphics[width=0.19\textwidth]{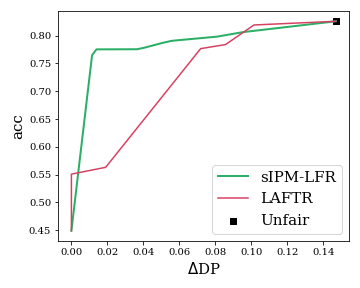}
    \includegraphics[width=0.19\textwidth]{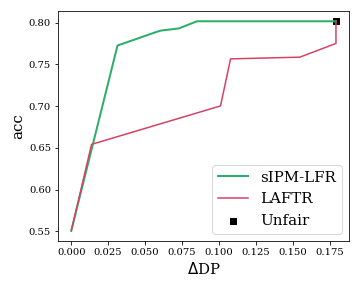}
}
\centerline{
    \includegraphics[width=0.19\textwidth]{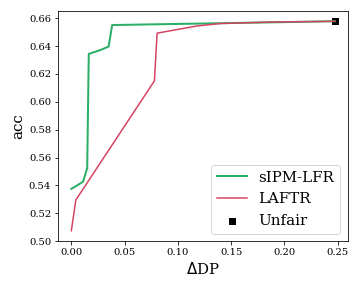}
    \includegraphics[width=0.19\textwidth]{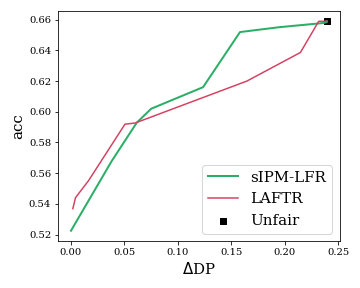}
    \includegraphics[width=0.19\textwidth]{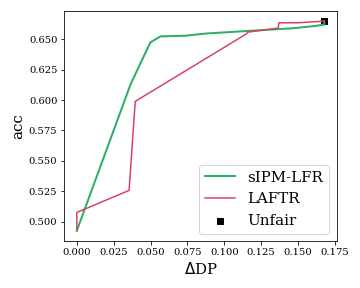}
    \includegraphics[width=0.19\textwidth]{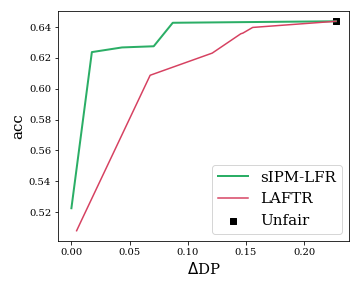}
    \includegraphics[width=0.19\textwidth]{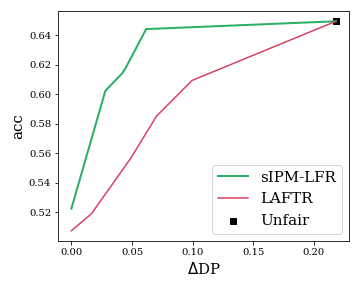}
}
\centerline{
    \includegraphics[width=0.19\textwidth]{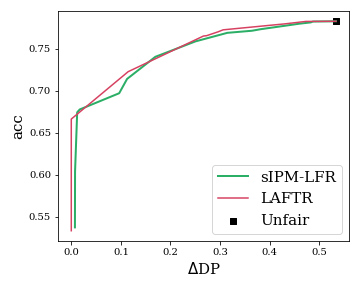}
    \includegraphics[width=0.19\textwidth]{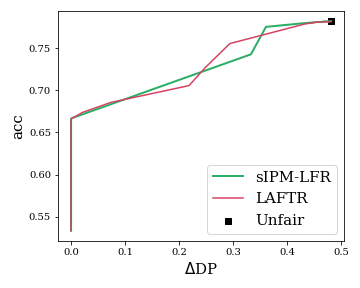}
    \includegraphics[width=0.19\textwidth]{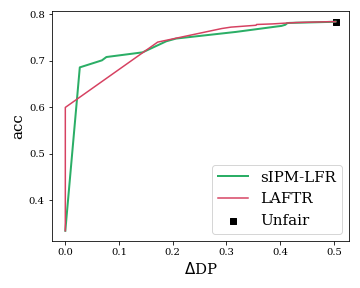}
    \includegraphics[width=0.19\textwidth]{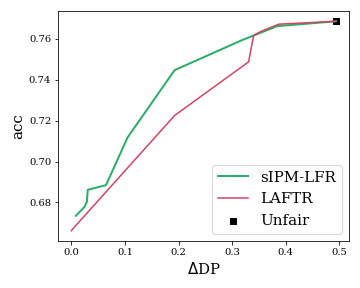}
    \includegraphics[width=0.19\textwidth]{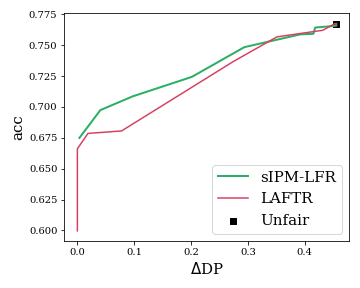}
}
\caption{Unsupervised LFR: Pareto-front lines between $\Delta \texttt{DP}$ and \texttt{acc} on the test data of (top) \textit{Adult}, (middle) \textit{COMPAS}, and (bottom) \textit{Health}.
The results of the five prediction models are given: 
(left to right)  linear, RBF-SVM, 1-LeakyReLU-NN, 1-Sigmoid-NN, and 2-Sigmoid-NN.
}
\label{fig:unsup}
\end{center}
\vskip -0.2in
\end{figure*}

\subsection{Extension to other fairness measures}
{The parametric IPM for DP can be easily extended to other group fairness measures such as the equal opportunity (EOpp) or equalized odds (EO). 
Let $\mbP_{s,y}^h$ be the distribution of $\bZ|S=s,Y=y$ for $s\in\{0,1\}$ and $y\in\{0,1\}$.
For a given function $\phi$ and a prediction function $f$, the fairness levels of EOpp and EO are defined as 
$$
EOpp_{\phi}(f)=|\E_{Z\sim\mathbb{P}_{0,0}^h}[\phi\circ f(Z)]-\E_{Z\sim\mathbb{P}_{1,0}^h}[\phi\circ f(Z)]|
$$
and
$$EO_{\phi}(f)=\sum_{y\in\{0,1\}}|\E_{Z\sim\mathbb{P}_{0,y}^h}[\phi\circ f(Z)]-\E_{Z\sim\mathbb{P}_{1,y}^h}[\phi\circ f(Z)]|.$$ 
}

{
Note that the main result of the previous section is to characterize the relationship between $d_{\cV_{sig}} (\mbP_{0}, \mbP_{1})$ and  $|\E_{Z\sim\mathbb{P}_0}[\phi\circ f(Z)]-\E_{Z\sim\mathbb{P}_1}[\phi\circ f(Z)]|$ for given two distributions $\mbP_{0}$ and $\mbP_{1}.$
We can derive similar theoretical results for EOpp and EO simply by letting $\mbP_{s}$ to $\mbP_{s,y}^h$. 
If we let $\mbP_{s}=\mbP_{s,0}^h$, 
we would obtain the connection between $d_{\cV_{sig}} (\mbP_{0,0}^h, \mbP_{1,0}^h)$ and $EOpp_\phi(f)$. 
Similarly, we could obtain the connection between
$\sum_{y\in\{0,1\}} d_{\cV_{sig}} (\mbP_{0,y}^h, \mbP_{1,y}^h)$
and $EO_{\phi}(f).$ For learning $f$ and $h$ for EOpp and EO, 
we minimize (\ref{eq:obj-sup}) and (\ref{eq:obj-unsup})
after replacing $d_{\cV_{sig}} (\mbP_{0}^h, \mbP_{1}^h)$
with $d_{\cV_{sig}} (\mbP_{0,0}^h, \mbP_{1,0}^h)$ and 
$\sum_{y\in\{0,1\}} d_{\cV_{sig}} (\mbP_{0,y}^h, \mbP_{1,y}^h),$
respectively.}

\section{Experiments}
\label{exps}

This section empirically shows that LFR using the sigmoid IPM (sIPM-LFR) performs well by analyzing supervised and unsupervised LFR tasks. 
Among these two tasks, we focus more on the latter because it is the case where fair representations is more important.
We show that the sIPM-LFR yields better and more stable performances than other baselines. 
For unsupervised LFR, in particular, the representations generated by our method usually give improved prediction accuracies for various downstream tasks.

We also do several ablation studies for the sIPM-LFR algorithm,
where the results for the stability issue are
reported in the main manuscript, and the others are presented in Appendix \ref{appendix:abl}. 
We here inform that we obtain the results of the baseline algorithms
of LFR by our own experiments (i.e. not copied from the related literature)
and report the averaged results from five random implementations. 

\subsection{Experimental setup}
\label{exps_setup}
\paragraph{Datasets}

We analyze three benchmark datasets - 1) \textit{Adult} \cite{Dua:2019}, 2) \textit{COMPAS} \footnote{https://github.com/propublica/compas-analysis}, 
and 3) \textit{Health} \footnote{https://foreverdata.org/1015/index.html},
which are analyzed in 
\citet{pmlr-v28-zemel13,9aa5ba8a091248d597ff7cf0173da151,Madras2018LearningAF,ruoss2020learning} for LFR.

\textit{Adult} contains personal information of over 40,000 individuals from the 1994 US Census.
The label indicates whether each person's income is over 50K\$ or not, and the sensitive variable is gender information.

\textit{COMPAS} contains criminal information of over 5,000 individuals from Florida.
The label is whether each person commits recidivism within two years, and the sensitive variable is race information.

\textit{Health} contains hospitalization records and insurance claims of over 60,000 patients.
The label is the binary Charlson index that estimates the death risk in the future ten years, and the sensitive variable is the binarized age information with a threshold of 70.
\textit{Health} also has tens of auxiliary binary labels, called the primary condition group (PCG) labels, indicating patients' insurance claim to the specific medical conditions, which can be utilized to conduct further downstream classification tasks. 
In our experiments, five auxiliary labels that are commonly
used in related literature are analyzed. 

We split the whole data into training and test data randomly, except for \textit{Adult} which already consists of training and test data.
We split the training data once more into two parts of the ratio 80\% and 20\%, each of which is used for training and validation, respectively. 
See Appendix \ref{appendix:exp_setup} for more detailed descriptions of the datasets including their pre-processing procedures.

\paragraph{Architectures}
We set up the architecture construction scheme similar to other works for LFR \cite{9aa5ba8a091248d597ff7cf0173da151,Madras2018LearningAF}. 
The  architecture of the encoder is fixed to a single-layered neural network with the LeakyReLU activation and we consider the value of $m$ as 60, 8, and 40 for \textit{Adult}, \textit{COMPAS}, and \textit{Health}, respectively. 
Regarding the prediction model $f$, while only single-layered neural network with the LeakyReLU activation (1-LeakyReLU-NN)
is used for the supervised LFR, 
we take four more prediction models into account
for unsupervised LFR. That is,  
we consider five prediction models in total: (i) linear, (ii) SVM with RBF kernel, (iii) 1-LeakyReLU-NN, (iv) 1-Sigmoid-NN, and (v) 2-Sigmoid-NN, where 
the last two models  stand for single-layered and two-layered neural networks with the sigmoid activation, respectively.

\paragraph{Implementation details}

We refer to other related studies \cite{9aa5ba8a091248d597ff7cf0173da151,Madras2018LearningAF}  for overall implementation options. 
To solve the supervised LFR, we train the encoder $h$ and classifier $f$ by applying the stochastic gradient descent step to the objective function (\ref{eq:obj-sup}) for 400 training epochs, and the best networks are chosen based on the value of the difference between accuracy and level of DP-fairness, i.e., $\texttt{acc}-\Delta\texttt{DP}$, on validation data. 

For the unsupervised LFR, we first minimize the formula (\ref{eq:obj-unsup}) to optimize the encoder and decoder for 300 training epochs. 
From the encoder-decoder pairs obtained at each epoch, we select the best one with the minimum validation loss. 
Afterward, for given label information $Y,$ 
we train and select the best downstream classifier by minimizing the standard cross-entropy loss for 100 epochs while freezing the encoder. 

Following what \citet{2020alg} did, for all cases, we apply the Adadelta \cite{journals/corr/abs-1212-5701} optimizer with a learning rate of 2.0 and a mini-batch size of 512. 
More detailed descriptions including our pseudo algorithm are in Appendix \ref{appendix:exp_setup}.

\paragraph{Evaluation metric} 
We assess the trade-off between the prediction accuracy (\texttt{acc}) and level of DP-fairness which are summarized by Pareto-front graphs and tables. 
For the fairness measure, we mainly deal with the original DP
denoted by
$\Delta \texttt{DP}$ and also consider other variants such as MDP denoted by $\Delta \texttt{MDP}.$ 
See Appendix \ref{appendix:fairmeasures} for formulas of the other fairness measures we consider. 

\begin{table}[t]
	\caption{Unsupervised LFR: \texttt{acc} and $\Delta \texttt{DP}$ for downstream classification tasks on 
	five PCG labels in \textit{Health}.  
	We use the RBF-SVM for the prediction model. 
	}
	\label{table:pcg}
	\vskip 0.15in
	\begin{center}
		\begin{small}
			\begin{tabular}{l|l||c|c|c}
				\toprule
				Target label &  & Unfair & LAFTR & sIPM-LFR \checkmark \\
				\midrule
				\midrule
				\multirow{2}{*}{MSC2A3} & \texttt{acc} & 0.665 & 0.642 & 0.646 \\
				& $\Delta \texttt{DP}$ & 0.110 & 0.103 & \textbf{0.055} \\
				\midrule
				\multirow{2}{*}{METAB3} & \texttt{acc} & 0.669 & 0.662 & 0.664 \\
				& $\Delta \texttt{DP}$ & 0.093 & 0.091 & \textbf{0.084} \\
				\midrule
				\multirow{2}{*}{ARTHSPIN} & \texttt{acc} & 0.695 & 0.690 & 0.692 \\
				& $\Delta \texttt{DP}$ & 0.062 & 0.047 & \textbf{0.036} \\
				\midrule
				\multirow{2}{*}{NEUMENT} & \texttt{acc} & 0.759 & 0.730 & 0.728 \\
				& $\Delta \texttt{DP}$ & 0.302 & 0.170 & \textbf{0.138} \\
				\midrule
				\multirow{2}{*}{RESPR4} & \texttt{acc} & 0.730 & 0.727 & 0.727 \\
				& $\Delta \texttt{DP}$ & 0.011 & 0.009 & \textbf{0.003} \\
				\bottomrule
			\end{tabular}
		\end{small}
	\end{center}
	\vskip -0.1in
\end{table}

\begin{figure}[ht]
\vskip 0.2in
\begin{center}
\centerline{
    \includegraphics[width=0.24\textwidth]{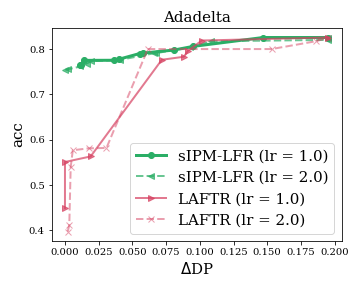}    \includegraphics[width=0.24\textwidth]{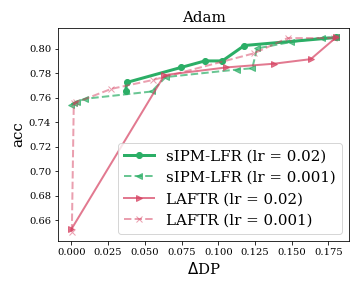}
}
\centerline{
    \includegraphics[width=0.24\textwidth]{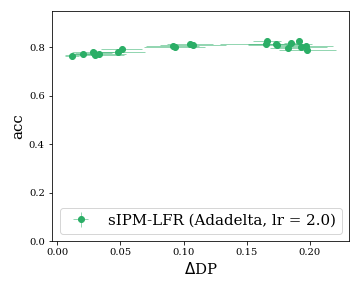}    \includegraphics[width=0.24\textwidth]{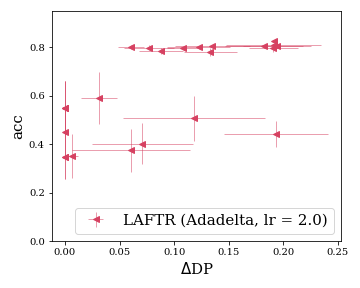}
}
\caption{ 
(Upper) Pareto-front lines between $\Delta\texttt{DP}$ and \texttt{acc} with various learning options. 
(Lower) Scatter plot with standard error bar of $\Delta\texttt{DP}$ and \texttt{acc} with various $\lambda$. 
Each horizontal and vertical bars present the standard errors
for $\Delta\texttt{DP}$ and \texttt{acc}, respectively.
All results are from \textit{Adult} test dataset.}
\label{fig:stable}
\end{center}
\vskip -0.2in
\end{figure}

\subsection{Supervised learning case}

We first evaluate our method in supervised LFR tasks and compare with other baselines including one LFR approach of \citet{Madras2018LearningAF} (i.e., LAFTR) and two non-LFR approaches in \citet{chuang2021fair}. 
Figure \ref{fig:sup} presents the Pareto-front trade-off graphs between the level of fairness ($\Delta \texttt{DP}$ and $\Delta \texttt{MDP}$) and \texttt{acc} on each test data. See Appendix \ref{appendix:sup} for 
the results of other fairness measures. 

We can clearly see that the proposed sIPM-LFR is compared favorably to the LAFTR even though a 
much simpler class of discriminators is used.
The results amply confirm our theoretical results that
the sigmoid IPM is sufficient for learning  
representations that are fair and good for prediction 
simultaneously.

It is also interesting to see that the sIPM-LFR is competitive
to the two non-LFR algorithms which learn a fair prediction
model without learning a representation. That is, the learned fair representation does not lose much information about the label. That is, the sIPM-LFR successively learns a good fair representation.


\subsection{Unsupervised learning case}

We show that the unsupervised sIPM-LFR provides fair representations of high quality that suit various subsequent downstream supervised tasks. 
As mentioned in Section \ref{exps_setup}, we first train an encoder by minimizing the objective function (\ref{eq:obj-unsup}) without label information of $Y,$ and then train the prediction model with label information
while freezing the encoder.

Figure \ref{fig:unsup} shows the Pareto-front lines
between \texttt{acc} and $\Delta \texttt{DP}$
of various prediction models on the three datasets. 
See Appendix \ref{appendix:exp} for the Pareto-front results for other fairness measures.

From the results, we can conclude that the sIPM-LFR is desirable to learn fair representations applicable better to various downstream tasks.
In particular, for \textit{COMPAS} the sIPM-LFR consistently gives superior results with large margins for all of the 5 prediction models. 
The superiority of the sIPM-LFR regardless of the final prediction model
supports our theoretical results that the sigmoid IPM
can control the level of fairness well for a large class
of prediction models.

We conduct further downstream classification tasks on \textit{Health} using five auxiliary PCG labels, whose results are summarized in Table \ref{table:pcg}.
We measure the level of DP-fairness while fixing the accuracies at certain levels.
It is obvious that the sIPM-LFR consistently achieves lower levels of DP-fairness than the other baselines do, again confirming the superiority of our method.

We also conduct experiments about visualization of the representation distributions and downstream classification with artificial labels. 
We report the results in Appendix \ref{appendix:exp}.

\paragraph{Experiments with additional datasets}
{Recently, there have been some concerns about the validity of widely-used benchmark datasets in the fair AI domain \cite{https://doi.org/10.48550/arxiv.2108.04884, bao2021its}. 
To answer this concern, we evaluate the sIPM-LFR on two additional datasets: \textit{ACSIncome} and \textit{Toxicity}. 
\textit{ACSIncome} is a pre-processed version of \textit{Adult}, and \textit{Toxicity} is a language dataset containing a large number of Wikipedia comments with ratings of toxicity. 
For \textit{Toxicity}, we generate the embedding vectors obtained by the BERT \cite{devlin-etal-2019-bert} and regard them as input vectors. 
For the detailed descriptions of those datasets and implementations, see Appendix \ref{app}.}

{
Table \ref{table:additional} shows that for a fixed prediction performance, the sIPM-LFR achieves lower levels of DP-fairness with large margins on the both datasets. 
We present more results for various $\lambda$ values and various prediction models in Appendix \ref{app}.
}

\begin{table}[h]
\caption{Unsupervised LFR: \texttt{acc} and $\Delta \texttt{DP}$ for downstream classification tasks on \textit{ASSIncome} and \textit{Toxicity}.
	We use the 1-Sigmoid-NN for the prediction model. }
	\label{table:additional}
	\vskip -0.1in
	\begin{center}
		\begin{small}
			\begin{tabular}{l|l||c|c|c}
				\toprule
				\multicolumn{2}{l||}{Data \scriptsize{(1-Sigmoid-NN)}}  & Unfair & LAFTR & sIPM-LFR \checkmark \\
				\midrule
				\midrule
				\multirow{2}{*}{\textit{ACSIncome}} & \texttt{acc} & 0.716 & 0.694  & 0.695  \\
				& $\Delta \texttt{DP}$ & 0.135 & 0.027 & \textbf{0.017} \\
				\midrule
				\multirow{2}{*}{\textit{Toxicity}} & \texttt{acc} & 0.802 & 0.790  & 0.790  \\
				& $\Delta \texttt{DP}$ & 0.042  & 0.021  & \textbf{0.013} \\
				\bottomrule
			\end{tabular}
		\end{small}
	\end{center}
	\vskip -0.25in
\end{table}

\subsection{Stability issue}

Compared to other adversarial LFR approaches, the learning procedure 
of the sIPM-LFR is numerically more stable. 
We demonstrate this advantage with two additional experiments, whose results are summarized in Figure \ref{fig:stable}.
The two plots at the first row of Figure \ref{fig:stable}
are the Pareto-front lines of the sIPM-LFR and LAFTR
for two optimizers and two learning rates on \textit{Adult}. 
It is noticeable that the results of the LAFTR are quite different for different learning rates when
the optimizer \textit{Adam} is used. In contrast, the results of the sIPM-LFR are stable.
This stability would be partly because the sIPM-LFR is simpler and thus less vulnerable to bad local minima.

The two plots at the second row are the scatter plots of ($\Delta\texttt{DP}$, \texttt{acc})
for various values of the regularization parameters for \textit{Adult}. 
There are many bad solutions observed for the LAFTR while the results for the sIPM-LFR vary smoothly.
These results confirm again that the sIPM-LFR is easier to learn good fair representation.

\section{Conclusion}
\label{conclusion}

In this paper, we devised a simple but powerful LFR method based on the sigmoid IPM called the sIPM-LFR. 
We proved that the sIPM-LFR can control the level of DP-fairness for a large class of prediction models
by controlling the fairness of the representation measured by the proposed parametric IPM.
We demonstrated that our learning method is competitive or better than other baselines, especially for unsupervised learning tasks, and is also numerically stable. 

{We note that any bounded, increasing, and measurable function instead of the sigmoid can be used and similar theoretical results can be derived. We focused on the sigmoid IPM in this paper because the sigmoid is popularly used in machine learning societies.} 

There are various directions for future works. Theoretically, the level of DP-fairness for diverse classes of functions other than the RKHS with the RBF kernel would be worth pursuing.
Also, it would be interesting to investigate other parametric IPMs which have similar properties to the sigmoid IPM.

It would also be interesting to apply the parametric IPM to other AI tasks, such as the generation of tabular data. Unlike image data, it is presumable that tabular data have 
a relatively smooth distribution.  
In this case, we conjecture that the parametric IPM would be enough to measure the similarity of two tabular data, which we will pursue in the near future.

\section*{Acknowledgements}
This work was supported by 
Institute of Information \& communications Technology Planning \& Evaluation(IITP) grant funded by the Korea government(MSIT) [No. 2019-0-01396, Development of framework for analyzing, detecting, mitigating of bias in AI model and training data], 
and supported by
Institute of Information \& communications Technology Planning \& Evaluation(IITP) grant funded by the Korea government(MSIT) [No. 2022-0-00184, Development and Study of AI Technologies to Inexpensively Conform to Evolving Policy on Ethics].


\bibliography{pfr}
\bibliographystyle{icml2022}

\newpage
\appendix
\onecolumn

\counterwithin{figure}{section}
\counterwithin{table}{section}

{\Large \textbf{Appendix}}

Appendix \ref{appendix:proofs} provides the rigorous proofs of theoretical results in Sections \ref{not_pre} and \ref{theory}. 
Also, we include an additional theoretical result that the sigmoid IPM can ensure more general types of DP-fairness if the prediction model is simple. 
The formulas of various fairness measures we consider are listed in Appendix \ref{appendix:fairmeasures}, and the detailed settings for the experiments
are explained in Appendix \ref{appendix:exp_setup}.
The results of additional experiments are presented in Appendix \ref{appendix:exp}.

\numberwithin{equation}{section}

\section{Theoretical proofs}\label{appendix:proofs}
\renewcommand{\theequation}{A.\arabic{equation}}

\subsection{Proofs of Proposition \ref{pro:1}, Theorem \ref{thm1}, and Theorem \ref{thm3_1}}

In this subsection, we let $\bZ_0$ and $\bZ_1$ are random vectors following the distributions $\mathbbm{P}^h_0$ and $\mathbbm{P}^h_1$, respectively. 
We start with the following lemma which plays a key role
in the other proofs.

\begin{lemma} \label{lemma1}
	For any $\epsilon>0$, there exists $c >0$ not depending on $\epsilon$ such that for any two probability measures $\mathbb{P}_{0}$ and $\mathbb{P}_{1}$ defined on $\mathbb{R}^m$,
	\begin{align*} 
	    d_{\mathcal{V}_{sig}}(\mathbb{P}_{0}, \mathbb{P}_{1}) < \epsilon
	\end{align*}
	implies
	\begin{align*}
	\sup_{\mathbf{a} \in \mathbb{R}^m} \sup_{t\in\mathbb{R}} \left| \mathbbm{P}\left( \mathbf{a}^{\top} \mathbf{U}_{0} \leq t \right)
	- \mathbbm{P}\left( \mathbf{a}^{\top} \mathbf{U}_{1} \leq t \right) \right| <& c \epsilon.
	\end{align*}
	where $\bU_0$ and $\bU_1$ are random vectors following the distributions $\mathbbm{P}_0$ and $\mathbbm{P}_1$, respectively. 
\end{lemma}

\begin{proof}
Fix $\epsilon>0$ and
 $\mathbf{a} \in \mathbb{R}^m$. We first consider the value of $t\in\mathbb{R}$ that the random variables $\mathbf{a}^{\top} \mathbf{U}_{0}$ and $\mathbf{a}^{\top} \mathbf{U}_{1}$ do not have a point mass at $t.$ 
 Then, there exists a small $\delta$ with $0<\delta < \min(\frac{1}{\log(1/\epsilon)}, \frac{1}{\log(1/10)})$ such that
	\begin{equation}
	\label{eq:1}
	\begin{split}
	\mathbbm{P} \left( \mathbf{a}^{\top} \mathbf{U}_{0} \in [t - \delta , t + \delta]\right) &< \epsilon \\
	\mathbbm{P} \left( \mathbf{a}^{\top} \mathbf{U}_{1} \in [t - \delta , t + \delta]\right) &< \epsilon
	\end{split}
	\end{equation}
	hold. 
	By the definition of $d_{\mathcal{V}_{sig}}(\mathbb{P}_{0}, \mathbb{P}_{1}) < \epsilon$, 
	we have 
	\begin{align}
	\label{eq:2}
	\left| \E \left[\sigma\left(\frac{\mathbf{a}^{\top} \mathbf{U}_{0} - t}{\delta^2}\right) \right]
	- \E \left[\sigma\left(\frac{\mathbf{a}^{\top} \mathbf{U}_{1} - t}{\delta^2}\right)\right]  \right| < \epsilon.
	\end{align} 
	On the other hand, for any $z\in\mathbb{R}$, the following inequality holds:
	\begin{align*}
	\frac{1}{1 + e^{-\frac{1}{\delta}}} \cdot \mathbb{I}\left( z > t + \delta\right)  
	\leq \sigma\left(\frac{z - t}{\delta^2}\right)
	\leq 1 - \frac{1}{1 + e^{-\frac{1}{\delta}}} \cdot \mathbb{I}\left( z \leq t - \delta\right).
	\end{align*}
	Thus for $s=0,1$ we have
	\begin{align}
	\label{eq:3}
	\frac{1}{1 + e^{-\frac{1}{\delta}}} \mathbbm{P}\left(\mathbf{a}^{\top}\bU_s>t+\delta\right)
	\le\E \left[\sigma\left(\frac{\mathbf{a}^{\top} \mathbf{U}_{s} - t}{\delta^2}\right) \right]
	\le 1-\frac{1}{1 + e^{-\frac{1}{\delta}}} \mathbbm{P} \left(\mathbf{a}^{\top}\bU_s\le t-\delta\right).
	\end{align}
	Also, from (\ref{eq:1}), we can bound the difference of the upper and lower bounds in (\ref{eq:3}):
	\begin{align}
	\label{eq:4}
	\begin{split}
	1-\frac{1}{1 + e^{-\frac{1}{\delta}}}\left\{
	\mathbbm{P} \left(\mathbf{a}^{\top}\bU_s \leq t-\delta\right)+
	\mathbbm{P}\left(\mathbf{a}^{\top}\bU_s>t+\delta\right)
	\right\}&\le 1-\frac{1}{1 + e^{-\frac{1}{\delta}}}(1-\epsilon)\\
	&\le 1-\frac{1-\epsilon}{1+\epsilon}\\
	&\le 2\epsilon.
	\end{split}
	\end{align}
	In turn, from (\ref{eq:3}) and (\ref{eq:4}), we have 
	\begin{align}
	\label{eq:5}
    \left|\E \left[\sigma\left(\frac{\mathbf{a}^{\top} \mathbf{U}_{s} - t}{\delta^2}\right) \right] - \left(1-
    \frac{1}{1 + e^{-\frac{1}{\delta}}}\mathbbm{P}\left(\mathbf{a}^{\top}\bU_s\le t-\delta\right)\right)\right|
    \le 2\epsilon.
	\end{align}
Therefore, by (\ref{eq:1}), (\ref{eq:2}), and (\ref{eq:5}), we obtain the following inequality:
	\begin{align*}
	\label{eq:6}
	\begin{split}
\frac{1}{1 + e^{-\frac{1}{\delta}}}\left| \mathbbm{P}\left( \mathbf{a}^{\top} \mathbf{U}_{0} \leq t \right)
	- \mathbbm{P}\left( \mathbf{a}^{\top} \mathbf{U}_{1} \leq t \right) \right| 
	&= \left|\left(1-\frac{1}{1 + e^{-\frac{1}{\delta}}}\mathbbm{P}\left( \mathbf{a}^{\top} \mathbf{U}_{0} \leq t \right)\right)-
	\left(1-\frac{1}{1 + e^{-\frac{1}{\delta}}}\mathbbm{P}\left( \mathbf{a}^{\top} \mathbf{U}_1 \leq t \right)\right)\right|\\
	&\le \frac{1}{1 + e^{-\frac{1}{\delta}}}\sum_{s\in\{0,1\}}\mathbbm{P} \left( \mathbf{a}^{\top} \mathbf{U}_s \in [t - \delta , t]\right)\\
	&+\sum_{s\in\{0,1\}}
	\left|\E \left[\sigma\left(\frac{\mathbf{a}^{\top} \mathbf{U}_{s} - t}{\delta^2}\right) \right] - \left(1-
    \frac{1}{1 + e^{-\frac{1}{\delta}}}\mathbbm{P}\left(\mathbf{a}^{\top}\bU_s\le t-\delta\right)\right)\right|\\
    &+
    \left| \E \left[\sigma\left(\frac{\mathbf{a}^{\top} \mathbf{U}_{0} - t}{\delta^2}\right) \right]
	- \E \left[\sigma\left(\frac{\mathbf{a}^{\top} \mathbf{U}_{1} - t}{\delta^2}\right)\right]  \right|\\
	&\le \frac{2\epsilon}{1 + e^{-\frac{1}{\delta}}}+5\epsilon,
	\end{split}
	\end{align*}
	which completes the proof.

For the case where either $\mathbf{a}^{\top}\bU_0$ or $\mathbf{a}^{\top}\bU_1$ has 
a point mass at $t$, we can construct a sequence $\{t_j\}_{j=1}^\infty$ such that 1) $t_j\downarrow t$ and 2) neither $\mathbf{a}^{\top}\bU_0$ nor $\mathbf{a}^{\top}\bU_1$ has a point mass at $\{t_j\}_{j=1}^\infty$. As $\mathbbm{P}\left( \mathbf{a}^{\top} \mathbf{U}_{s} \leq \cdot \right)$ is right-continuous, the following holds:
\begin{align*}
\lim_{j\to\infty}\left| \mathbbm{P}\left( \mathbf{a}^{\top} \mathbf{U}_{0} \leq t_j \right)
	- \mathbbm{P}\left( \mathbf{a}^{\top} \mathbf{U}_{1} \leq t_j \right) \right|
	=
	\left| \mathbbm{P}\left( \mathbf{a}^{\top} \mathbf{U}_{0} \leq t \right)
	- \mathbbm{P}\left( \mathbf{a}^{\top} \mathbf{U}_{1} \leq t \right) \right|<c\epsilon,     
\end{align*}
and the proof is done.
\end{proof}

\paragraph{Proof of Proposition \ref{pro:1}}
Let $\mathbf{U}_0$ and $\mathbf{U}_1$ are two random vectors whose distributions are $\mathbbm{P}_0$ and $\mathbbm{P}_1$, respectively.
    
    $(\implies)$ From $d_{\mathcal{V}_{sig}}(\mathbb{P}_{0}, \mathbb{P}_{1}) = 0,$
    we have 
    \begin{align*}
	\sup_{\mathbf{a} \in \mathbb{R}^m} \sup_{t} \left| \mathbbm{P}\left( \mathbf{a}^{\top} \mathbf{U}_{0} \leq t \right)
	- \mathbbm{P}\left( \mathbf{a}^{\top} \mathbf{U}_{1} \leq t \right) \right| = 0
	\end{align*}
	by Lemma \ref{lemma1}.
	Hence we have $\mathbf{a}^{\top} \mathbf{U}_{0} \overset{d}{=} \mathbf{a}^{\top} \mathbf{U}_{1}$ holds for all $\mathbf{a} \in \mathbb{R}^m$, which implies $\mathbb{P}_{0} \equiv \mathbb{P}_{1}$ due to the uniqueness of the characteristic function.
    
    $ (\impliedby)$ It is trivial since for any $v\in\mathcal{V}_{sig}$, we have 
    $\int v(\bm{u}) (d\mathbbm{P}_{0}(\bm{u}) - d\mathbbm{P}_{1}(\bm{u}))  = 0.$ \qed

\paragraph{Proof of Theorem \ref{thm1}}
The proof is trivial by Proposition \ref{pro:1}.

\paragraph{Proof of Theorem \ref{thm3_1}}


Let 
\begin{align*}
\cF_{n,B}^{NN}= \Big\{ &f(\bz)=\sum_{k=1}^n v_k \sigma(\bm{a}_k^\top \bz+b_k)+v_0: v_0 \in \mathbb{R},
|v_k| \le B, \bm{a}_k\in \mathbb{R}^m, 
b_k\in \mathbb{R}, k=1,\ldots,n\Big\},
\end{align*}
for $n\in \mathbb{N}$ and $B>0.$
Since $\cZ$ is bounded, there exists $M>0$ such that $\cZ \subset [-M,M]^d$.
By Theorem 2.2 of \citet{yukich1995sup},
for any $f\in \cF_{a,C}$ and $n \in \mathbb{N}$, there exist $\bm{a}_k \in \mathbb{R}^m$, $b_k \in \mathbb{R}$, $|v_k| \leq B$ for $k \in \{1,\dots,n\}$ and $v_0 \in \mathbb{R}$ such that
\begin{align*}
    \sup_{\bz\in \cZ}\left|f(\bz) - \sum_{k=1}^n v_k \sigma(\bm{a}_k^\top \bz+b_k)-v_0\right| \leq \frac{C^{\prime}}{\sqrt{n}}
\end{align*}
for some constant $C^{\prime}>0$.
Thus, we have
\begin{align*}
    \left|\int \left(f(\bz) - \sum_{k=1}^n v_k \sigma(\bm{a}_k^\top \bz+b_k)-v_0 \right) d\mbP_0^h(\bz)\right|
    \leq& \frac{C^{\prime}}{\sqrt{n}}.
\end{align*}
A similar bound holds for $\mbP_1^h.$ Hence, by Proposition \ref{pro:1},
\begin{align*}
        \left| \int f(\bz) (d\mbP_0^h(\bz)-d\mbP_1^h(\bz)) \right| 
        \leq &  \sum_{k=1}^n \left| \int (v_k \sigma(\bm{a}_k^\top\bz+b_k)+v_0)(d\mbP_0^h(\bz)-d\mbP_1^h(\bz)) \right| + 2\frac{C^{\prime}}{\sqrt{n}}\\
        \leq & nB d_{\mathcal{V}_{sig}}(\mathbb{P}_{0}^{h}, \mathbb{P}_{1}^{h}) + 2\frac{C^{\prime}}{\sqrt{n}}\\
        \leq & C^{\prime\prime} d_{\mathcal{V}_{sig}}(\mathbb{P}_{0}^{h}, \mathbb{P}_{1}^{h})^{1/3}
\end{align*}
holds for some constant $C^{\prime\prime} > 0$ 
if we let $n = \lceil\frac{1}{d_{\mathcal{V}_{sig}}(\mathbb{P}_{0}^{h}, \mathbb{P}_{1}^{h})^{2/3}} \rceil,$
and thus the proof for $k=1$ is complete with $c_1=C^{\prime\prime}$.

For $k>1,$ note that $f^k\in \cF_{a^K, k a^{k-1}C}$ if $f\in \cF_{a,C}$
(see p 940 of \cite{barron1993universal}) and thus the proof can be done similarly.
\qed

\subsection{Proof of Theorem \ref{thm3}}

To prove Theorem \ref{thm3}, we need the following three Lemmas \ref{lemma2}, \ref{lemma3}, and \ref{lemma:inf}. 


\begin{lemma}\label{lemma2}
	For $\forall r_1, r_2 \in \mathbb{N}$, let $r:=r_1+r_2$ and $\lambda_i := -1 + \frac{2i}{r}$ for $i=0,1,\dots,r$.
	Then there exists a vector $(\beta_0, \beta_1, \ldots, \beta_r) \in \mathbb{R}^{r+1}$ such that
	\begin{equation}
	\label{lemmaA.2_1}
	\sum_{i=0}^{r} \beta_i (x+\lambda_i y)^r = x^{r_1} y^{r_2} 
	\end{equation}
	and
	\begin{equation}
	 \label{lemmaA.2_2}
	\sum_{i=0}^r |\beta_i| < e^r
	\end{equation}
 for all $x,y \in \mathbb{R}$.
\end{lemma}

\begin{proof}
We first find a closed form solution of $(\beta_0, \beta_1, \ldots, \beta_r)$ of (\ref{lemmaA.2_1}) and then show that it satisfies (\ref{lemmaA.2_2}).
    Note that if a vector $(\beta_0, \beta_1, \dots, \beta_r) \in \mathbb{R}^{r+1}$ satisfies
	\begin{align*}
	\sum_{i=0}^r \beta_i \lambda_i^{k} &= 0 \text{ for } k \in \{0,1,\dots, r\} \setminus \{r_2\}\\
	\end{align*}
	and
	\begin{align*}
	\sum_{i=0}^r \beta_i \lambda_i^{r_2} &= \frac{1}{{r \choose r_2}},
	\end{align*}
	then it is a solution of (\ref{lemmaA.2_1}). 
	Let V be the Vandermonde matrix defined as
	\begin{equation*}
	V = \begin{pmatrix}
	1 & 1 & \dots & 1 \\
	\lambda_0 & \lambda_1 & \dots & \lambda_r \\
	\dots & \dots & \ddots & \dots \\
	\lambda_0^r & \lambda_1^r & \dots & \lambda_r^r 
	\end{pmatrix} .
	\end{equation*}
	Using the Vandermonde matrix, the above two equations can be re-formulated as
	\begin{equation*}
	V \times (\beta_0, \beta_1, \dots, \beta_r)^{\top} = \frac{1}{{r \choose r_2}} \bm{e}_{r_2+1},
	\end{equation*}
	where $\bm{e}_{r_2+1} \in \mathbb{R}^{r+1}$ is the
	vector whose $(r_2+1)$-th element is 1 and the rests are 0. 
	Then by \citet{man2017computing}, it is known that $V^{-1}$ can be expressed as the product of two matrices $W$ and $A$, where the matrices $W$ and $A$ are given as  
	\begin{equation*}
	W=\left(\begin{array}{ccccc}
	\frac{\lambda_{0}{ }^{r}}{\prod_{j \neq 0}\left(\lambda_{0}-\lambda_{j}\right)} & \frac{\lambda_{0}{ }^{r-1}}{\prod_{j \neq 0}\left(\lambda_{0}-\lambda_{j}\right)} & \dots & \frac{1}{\prod_{j \neq 0}\left(\lambda_{0}-\lambda_{j}\right)} \\
	\frac{\lambda_{1}{ }^{r}}{\prod_{j \neq 1}\left(\lambda_{1}-\lambda_{j}\right)} & \frac{\lambda_{1}{ }^{r-1}}{\prod_{j \neq 1}\left(\lambda_{1}-\lambda_{j}\right)} & \dots & \frac{1}{\prod_{j \neq 1}\left(\lambda_{1}-\lambda_{j}\right)} \\
	\vdots & \vdots & \ddots & \vdots \\
	\frac{\lambda_{r}{ }^{r}}{\prod_{j \neq r}\left(\lambda_{r}-\lambda_{j}\right)} & \frac{\lambda_{r}{ }^{r-1}}{\prod_{j \neq r}\left(\lambda_{r}-\lambda_{j}\right)} & \dots & \frac{1}{\prod_{j \neq r}\left(\lambda_{r}-\lambda_{j}\right)}
	\end{array}\right)
	\textup{ and }
	A=\left(\begin{array}{ccccc}
	a_0 & 0 & 0 & \cdots & 0 \\
	a_{1} & a_0 & 0 & \cdots & 0 \\
	a_{2} & a_{1} & a_0 & \cdots & 0 \\
	\vdots & \vdots & \vdots & \ddots & \vdots \\
	a_{r} & a_{r-1} & a_{r-2} & \cdots & a_0
	\end{array}\right), 
	\end{equation*}
	where $a_0 = 1, a_{1}=-\sum \lambda_{j}, a_{2}=\sum_{j < l} \lambda_{j} \lambda_{l}, a_{3}=-\sum_{j < l < s} \lambda_{j} \lambda_{l} \lambda_{s}, \dots,$ and $a_{r+1}=(-1)^{r+1} \prod_{j=0}^r \lambda_{j}.$
	Since
	\begin{align*}
	(\beta_0, \beta_1, \dots, \beta_r)^{\top} = \frac{1}{{r \choose r_2}} V^{-1} \bm{e}_{r_2 + 1}
	= \frac{1}{{r \choose r_2}} W A \bm{e}_{r_2 + 1}
	= \frac{1}{{r \choose r_2}} W (0,\dots,0,a_0,a_1,\dots,a_{r_1})^{\top},
 	\end{align*}
    we obtain the closed form solution $\beta_i, i=0,1,\ldots,r$
    of (\ref{lemmaA.2_1}) given as
	\begin{align}
	\beta_i = \frac{1}{{r \choose r_2}\prod_{j \neq i}\left(\lambda_{i}-\lambda_{j}\right)} 
	\left(\lambda_i^{r_1} + \sum_{l=1}^{r_1} (-1)^l \lambda_{i}^{r_1 - l}
	\sum_{k_1 < \dots < k_{l}} \lambda_{k_1} \dots \lambda_{k_{l}}\right). \label{lemmaA.2_3}
	\end{align}
	
	Now we are going to show that the vector $(\beta_0, \beta_1, \dots, \beta_r)$ of (\ref{lemmaA.2_3}) satisfies (\ref{lemmaA.2_2}).
	The numerator of $\beta_i$ in (\ref{lemmaA.2_3}) can be rewritten as
	\begin{align*}
	\lambda_i^{r_i}+\sum_{l=1}^{r_1} (-1)^l \lambda_{i}^{r_1 - l}
	\sum_{k_1 < \dots < k_{l}} \lambda_{k_1} \dots \lambda_{k_{l}} =& 
	\lambda_i^{r_i}+\sum_{l=1}^{r_1} (-1)^l \lambda_{i}^{r_1 - l}
	\sum_{\underset{\{k_1, \dots, k_{l}\} \not\ni i}{k_1 < \dots < k_{l}}} \lambda_{k_1} \dots \lambda_{k_{l}}\\
	&+ \sum_{l=1}^{r_1} (-1)^l \lambda_{i}^{r_1 - l} \lambda_i \sum_{\underset{\{k_1, \dots, k_{l-1}\} \not\ni i}{k_1 < \dots < k_{l-1}}} \lambda_{k_1} \dots \lambda_{k_{l-1}} \\
	&=\lambda_i^{r_i}+\sum_{l=1}^{r_1} (-1)^l \lambda_{i}^{r_1 - l}
	\sum_{\underset{\{k_1, \dots, k_{l}\} \not\ni i}{k_1 < \dots < k_{l}}} \lambda_{k_1} \dots \lambda_{k_{l}}\\
	&-\lambda_i^{r_1}+(-1)\sum_{l=1}^{r_1-1} (-1)^l \lambda_{i}^{r_1 - l} \lambda_i \sum_{\underset{\{k_1, \dots, k_{l}\} \not\ni i}{k_1 < \dots < k_{l}}} \lambda_{k_1} \dots \lambda_{k_{l}}\\
	&= (-1)^{r_1} \sum_{\underset{\{k_1, \dots, k_{{r_1}}\} \not\ni i}
	{k_1 < \dots < k_{{r_1}}}} \lambda_{k_1} \dots \lambda_{k_{{r_1}}}.
	\end{align*}
	Thus, $\beta_i$ is given as
	\begin{equation*}
	\beta_i = \frac{(-1)^{r_1}}{{r \choose r_2}}
	\Big(\sum_{\underset{\{k_1, \dots, k_{{r_1}}\} \not\ni i}{k_1 < \dots < k_{{r_1}}}} \lambda_{k_1} \dots \lambda_{k_{{r_1}}}\Big)\Big/
	\Big(\prod_{j \neq i}\left(\lambda_{i}-\lambda_{j}\right)\Big).
	\end{equation*}
	Finally, we can find the lower bound of $|\prod_{j \neq i}\left(\lambda_{i}-\lambda_{j}\right)|$ given as
	\begin{align*}
	|\prod_{j \neq i}\left(\lambda_{i}-\lambda_{j}\right)| & \geq 
	\begin{cases}
	(\frac{r}{2})!(\frac{r}{2})! / (\frac{r}{2})^r ,& \mbox{if }r\mbox{ is even} \\
	(\frac{r+1}{2})!(\frac{r-1}{2})! / (\frac{r}{2})^r,& \mbox{if }r\mbox{ is odd}
	\end{cases} \\
	& > (r+1)e^{-r}. 
	\end{align*}
	The second inequality is derived by the inequality from the Stirling's approximation, that is,
	\begin{align*}
	    n!>\sqrt{2\pi n}\left( \frac{n}{e}\right)^n e^{\frac{1}{12n+1}}>\sqrt{2\pi n}\left( \frac{n}{e}\right)^n.
	\end{align*}
	Therefore, we finally obtain
	\begin{align*}
	\sum_{i=0}^{r} |\beta_i| 
	&<\frac{e^r}{(r+1){r \choose r_2}} \sum_{i=0}^{r} \sum_{\underset{\{k_1, \dots, k_{{r_1}}\} \not\ni i}{k_1 < \dots < k_{{r_1}}}} | \lambda_{k_1} \dots \lambda_{k_{{r_1}}}|\\
	&<\frac{e^r}{(r+1){r \choose r_2}} \sum_{i=0}^{r} \sum_{\underset{\{k_1, \dots, k_{{r_1}}\} \not\ni i}{k_1 < \dots < k_{{r_1}}}} 1^{r_1}\\
	&=\frac{e^r}{(r+1){r \choose r_2}}(r+1){r\choose r_2}=e^r,
	\end{align*}
	and the proof is done.
\end{proof}


\begin{lemma} \label{lemma3}
For $\forall u \in \mathbb{N}$ and $\forall r_1, r_2, \dots, r_u \in \mathbb{N}$, let $r=r_1 + \dots + r_{u}$.
Then there exist a real-valued sequence $\{\beta_i\}_{i=0}^{\infty}$ and a 2-dimensional array $\{\lambda_{ij}\}_{i\in\mathbb{N}_0,j\in\{1,\dots,u\}}$  each of whose elements is bounded by $[-1,1]$ such that
\begin{align*}
    \sum_{i} \left[ \beta_i ( \sum_{j=1}^{u} \lambda_{ij} z_j )^r \right] = z_{1}^{r_1} z_{2}^{r_2} \dots z_{u}^{r_u}
    \textup{ and }
    \sum_{i} |\beta_i| \leq e^{(u-1)r}.
\end{align*}
holds for all $z_1 , \dots , z_u \in \mathbb{R}$.

\end{lemma}
\begin{proof}
    We prove the lemma with the mathematical induction. The statement is obvious for $u=1,$ and we have shown in Lemma \ref{lemma2} that the statement also holds for $u=2$.
    Suppose that the statement holds for some $u=N-1\in\mathbb{N}$, and we will prove the statement is also valid when $N$.
    For given $r_1, r_2, \dots, r_{N-1} \in \mathbb{N}$, let $r=r_1 + \dots + r_{N-1}$. By the assumption, there exist $\beta^{\prime}_0, \beta^{\prime}_1, \dots \in \mathbb{R}$ and $\lambda^{\prime}_{ij} \in [-1,1] \text{ for } i \in \mathbb{N}_0$ and $j \in \{1,\dots,N-1\}$ such that
\begin{align*}
    \sum_{i} \left[ \beta^{\prime}_i ( \sum_{j=1}^{N-1} \lambda^{\prime}_{ij} z_j )^r \right] = z_{1}^{r_1} z_{2}^{r_2} \dots z_{N-1}^{r_{N-1}}
    \textup{ and }
    \sum_{i} |\beta^{\prime}_i| < e^{(N-2)r}.
\end{align*}
Note that by Lemma \ref{lemma2}, for any $r_{N}\in\mathbb{N}$ there exist $\beta^{\prime\prime}_0, \beta^{\prime\prime}_1, \dots, \beta^{\prime\prime}_{r+r_{N}} \in \mathbb{R}$ with $\sum_{k=0}^{r+r_{N}} |\beta^{\prime\prime}_k| < e^{r+r_{N}}$ and $\lambda^{\prime\prime}_k \in [-1,1]$ for $k \in \{0,1,\dots,(r+r_{N})\}$ such that
\begin{align*}
    z_{1}^{r_1} z_{2}^{r_2} \dots z_{N-1}^{r_{N-1}} z_{N}^{r_{N}} 
    =& \sum_{i} \left[ \beta^{\prime}_i ( \sum_{j=1}^{N-1} \lambda^{\prime}_{ij} z_j )^r z_{N}^{r_{N}} \right]\\
    =& \sum_{i} \left[ \beta^{\prime}_i \sum_{k=0}^{r+r_{N}} \left(\beta^{\prime\prime}_k(\sum_{j=1}^{N-1} \lambda^{\prime}_{ij} z_j + \lambda^{\prime\prime}_k z_{N})^{r+r_{N}}\right) \right]\\
    =& \sum_{i}  \sum_{k=0}^{r+r_{N}}  \left( \beta^{\prime}_i \beta^{\prime\prime}_k(\sum_{j=1}^{N-1} \lambda^{\prime}_{ij} z_j + \lambda^{\prime\prime}_k z_{N})^{r+r_{N}}\right).
\end{align*} 
Also, we can check that $\sum_{i}  |\sum_{k=0}^{r+r_N} \beta^{\prime}_i \beta^{\prime\prime}_k| < e^{(N-1)(r+r_{N})}$ holds. 
Thus, the statement holds for $N$ if we set $\beta_i=\sum_{k=1}^{r+r_{N}}\beta_i^{\prime}\beta_k^{\prime\prime},$ for  $i\in\mathbb{N}_0$ and  $\{\lambda_{ij}\}_{i\in\mathbb{N}_0,j\in\{1,\dots,N\}}$ accordingly.
\end{proof}


\begin{lemma}
\label{lemma:inf}
	Suppose that $d_{\cV_{sig}}(\mathbb{P}_0^h, \mathbb{P}_1^h) < \epsilon$ for a given $\epsilon > 0.$ Then for a $m$-dimensional index $\mathbf{r} := (r_1 , \dots , r_m)^{\top} \in \mathbb{N}_0^m$, there exist $c_1,c_2>0$ not depending on $\epsilon$ and $\mathbf{r}$ such that 
    \begin{align*}
\left| \mathbbm{E}(\mathbf{Z}_0^{\mathbf{r}}) - \mathbbm{E}(\mathbf{Z}_1^{\mathbf{r}}) \right|<
c_1 c_2^{{|\mathbf{r}|_1}} \epsilon\\
    \end{align*}
    holds where $\bZ_0$ and $\bZ_1$ are random vectors following the distributions $\mathbbm{P}_0^h$ and $\mathbbm{P}_1^h$, respectively. 
\end{lemma}

\begin{proof}
Since $\cZ$ is bounded, there exists $M>0$ such that $\cZ \subset [-M,M]^m$.
By Lemma \ref{lemma3}, there exist a real-valued sequence $\{\beta_i\}_{i=0}^{\infty}$ and a 2-dimensional array $\{\lambda_{ij}\}_{i\in\mathbb{N}_0,j\in\{1,\dots,m\}}$ that each element is bounded by $[-1,1]$ such that
\begin{align*}
    \sum_{i} \left[ \beta_i ( \sum_{j=1}^{m} \lambda_{ij} z_j )^{|\mathbf{r}|_1} \right] = z_{1}^{r_1} z_{2}^{r_2} \dots z_{m}^{r_m}
    \textup{ and }
    \sum_{i} |\beta_i| \leq e^{(m-1)|\mathbf{r}|_1}
\end{align*}
hold for all $\bz\in\cZ$. Thus, we have
\begin{align*}
    \left| \mathbbm{E}(\mathbf{Z}_0^{\mathbf{r}}) - \mathbbm{E}(\mathbf{Z}_1^{\mathbf{r}}) \right|
    \leq & \sum_{i} |\beta_i| \left| 
    \mathbbm{E}\left( ( \sum_{j=1}^{m} \lambda_{ij} \mathbf{Z}_{0j} )^{|\mathbf{r}|_1}\right)
    - \mathbbm{E}\left( ( \sum_{j=1}^{m} \lambda_{ij} \mathbf{Z}_{1j} )^{|\mathbf{r}|_1}\right)\right|  \\
    \leq & e^{(m-1) |\mathbf{r}|_1} \sup_{|\mathbf{a}|_{\infty}\leq1} \left| \mathbbm{E}\left( (\mathbf{a}^{\top} \mathbf{Z}_{0})^{|\mathbf{r}|_1} \right) 
    - \mathbbm{E}\left( (\mathbf{a}^{\top} \mathbf{Z}_{1})^{|\mathbf{r}|_1} \right) \right|.
\end{align*}
In addition, since  the function $s(\cdot):=(\cdot)^{|\mathbf{r}|_1}/({|\mathbf{r}|_1}(mM)^{{|\mathbf{r}|_1}-1})$ is
$1$-Lipschitz on $[-mM, mM]$, we have
\begin{align*}
    \sup_{|\mathbf{a}|_{\infty}\leq1} \left| \mathbbm{E}\left( (\mathbf{a}^{\top} \mathbf{Z}_{0})^{|\mathbf{r}|_1} \right) 
    - \mathbbm{E}\left( (\mathbf{a}^{\top} \mathbf{Z}_{1})^{|\mathbf{r}|_1} \right) \right|
    &\leq {|\mathbf{r}|_1}(mM)^{{|\mathbf{r}|_1}-1} \sup_{|\mathbf{a}|_{\infty}\leq1}
    \int_{t\in\mathbb{R}} \left| \mathbbm{P}\left( \mathbf{a}^{\top} \mathbf{Z}_{0} \leq t \right)
	- \mathbbm{P}\left( \mathbf{a}^{\top} \mathbf{Z}_{1} \leq t \right) \right| dt
\end{align*}
by the property of the Wasserstein metric \cite{gibbs2002choosing}. 
Now, by Lemma \ref{lemma1}, there exist $c_1,c_2>0$ such that 
\begin{align*}
    \left| \mathbbm{E}(\mathbf{Z}_0^{\mathbf{r}}) - \mathbbm{E}(\mathbf{Z}_1^{\mathbf{r}}) \right|
	\leq& c_1 {c_2}^{{|\mathbf{r}|_1}}
	\epsilon, 
\end{align*}
and the proof is done.
\end{proof}

\paragraph{Proof of Theorem \ref{thm3}}
By using Taylor's expansion, we can write
\begin{equation*}
    \begin{split}
        \phi_k(f(\mathbf{z}))=f(\mathbf{z})^{k} = \sum_{\bm{j} \in \mathbb{N}_0^m} \frac{a_{\bm{j}, k}}{\bm{j}!} \mathbf{z}^{\bm{j}},
    \end{split}
\end{equation*}
where $\bm{j} = (j_1 , \dots, j_m) \in \mathbb{N}_0^m$ and 
$a_{\bm{j}, k} = D^{\bm{j}}(f^k)|_{\bz=0}.$ 
For a given $\bm{j},$ we are going to inductively show that 
\begin{equation}\label{eq:induction}
    \begin{split}
        |a_{\bm{j},k}| \le \sqrt{\bm{j} !} ( kB )^{|\bm{j}|_1}
    \end{split}
\end{equation}
for all $k \in \mathbb{N}.$
The case when $k = 1$ is trivial due to the definition of $\mathcal{F}_{\cC^{\infty},B}.$ 
Suppose the equation (\ref{eq:induction}) holds when $k = N - 1 \ge 1$ for some $N.$
Then, for $k = N,$ 
$$ f(\mathbf{z})^{N} = f(\mathbf{z})^{N-1} f(\mathbf{z}) 
    = \left( \sum_{\bm{l} \in \mathbb{N}_0^m} \frac{a_{\bm{l}, N-1}}{\bm{l}!} \mathbf{z}^{\bm{l}} \right) \left( \sum_{\bm{m} \in \mathbb{N}_0^m} \frac{a_{\bm{m},1}}{\bm{m}!} \mathbf{z}^{\bm{m}} \right). $$
Thus, the absolute value of the $\mathbf{z}^{\bm{j}}$'s coefficient for $f(\mathbf{z})^{N}$ satisfies
\begin{align*}
    \left|\sum_{\bm{h} \in \mathbb{N}_0^m, \bm{h} \leq \bm{j}} \frac{a_{\bm{h}, k-1}}{\bm{h}!} \frac{a_{\bm{j}-\bm{h}, 1}}{(\bm{j}-\bm{h})!}\right|
    \le & \sum_{\bm{h} \in \mathbb{N}_0^m, \bm{h} \leq \bm{j}} \frac{ \sqrt{\bm{h} !} ((k-1)B)^{|\bm{h}|_1} }{\bm{h}!} \frac{ \sqrt{\bm{(j-h)} !} B^{|\bm{j}-\bm{h}|_1}}{(\bm{j}-\bm{h})!}\\
    = & \frac{B^{|\bm{j}|_1}}{\sqrt{\bm{j}!}} 
    \sum_{h_1 = 0}^{j_1} \dots \sum_{h_m = 0}^{j_m} \left(\sqrt{{j_1 \choose h_1}} (k-1)^{h_1}\right) \cdots \left(\sqrt{{j_m \choose h_m}} (k-1)^{h_m}\right)\\
    =  & \frac{B^{|\bm{j}|_1}}{\sqrt{\bm{j}!}} 
    \left(\sum_{h_1 = 0}^{j_1} \sqrt{{j_1 \choose h_1}} (k-1)^{h_1}\right) \cdots 
    \left(\sum_{h_m = 0}^{j_m}  \sqrt{{j_m \choose h_m}} (k-1)^{h_m}\right)\\
    \leq & \frac{B^{|\bm{j}|_1}}{\sqrt{\bm{j}!}} 
    \left(\sum_{h_1 = 0}^{j_1} {j_1 \choose h_1} (k-1)^{h_1}\right) \cdots 
    \left(\sum_{h_m = 0}^{j_m}  {j_m \choose h_m} (k-1)^{h_m}\right) \\
    =& \frac{\sqrt{\bm{j}!}(kB)^{|\bm{j}|_1}}{\bm{j}!},
\end{align*}
which implies that (\ref{eq:induction}) holds for all $\bm{j} \in \mathbb{N}_0^m$ and $k \in \mathbb{N}$.
Thus, we have
\begin{align*}
    \begin{split}
        \left|\int f(\mathbf{z})^{k} (d\mathbb{P}_{0}^h(\bz) - d\mathbb{P}_{1}^h(\bz)) \right|\le 
        &\sum_{\bm{j} \in \mathbb{N}_0^m} \left|\frac{a_{\bm{j}, k}}{\bm{j}!}\right| 
        \left|\int  \mathbf{z}^{\bm{j}} (d\mathbb{P}_{0}^h(\bz) - d\mathbb{P}_{1}^h(\bz))\right|\\
        \le& C_2 d_{\cV_{sig}}(\mathbb{P}_0^h,\mathbb{P}_1^h)
        \sum_{\bm{j} \in \mathbb{N}_0^m} \left|\frac{( C_1 kB )^{|\bm{j}|_1}}{\sqrt{\bm{j}!}}\right| \\
        \le& C_2 2^m \max(1, C_1 k B) d_{\cV_{sig}}(\mathbb{P}_0^h,\mathbb{P}_1^h)
        \sum_{\bm{j} \in \mathbb{N}_0^m} \left|\frac{( C_1 kB )^{2|\bm{j}|_1}}{\bm{j}!}\right| \\
        = & C_2 2^m \max(1, C_1 k B) d_{\cV_{sig}}(\mathbb{P}_0^h,\mathbb{P}_1^h)
        \exp( (C_1 kB)^{2} )
    \end{split}
\end{align*}
for some $C_1$, $C_2>0,$ where the second and third inequalities are due to Lemma \ref{lemma:inf} and $\sqrt{\bm{j}!} > (\lfloor \bm{j}/2 \rfloor)!,$ respectively, which completes the proof.
\qed


\subsection{Proof of Proposition \ref{prop:rbf}}

\paragraph{Proof of Proposition \ref{prop:rbf}}
    The proof is a slight modification of the proof of Theorem 4.48 in \citet{steinwart2008support}.
    Let $V_{\gamma} : L_2 (\cZ) \to \mathcal{H}_{\gamma}(\cZ)$ be the metric surjection defined by
    \begin{align*}
        V_{\gamma} g(\bm{z})=\frac{2^{m/2}}{\gamma^{m/2} \pi^{m/4}} \int_{\mathbb{R}^{m}} e^{-2 \gamma^{-2}\|\bm{z}-\bm{y}\|_{2}^{2}} g(\bm{y}) d \bm{y}
    \end{align*}
    for $g \in L_{2}\left(\mathbb{R}^{m}\right)$ and 
    $ \bm{z} \in \cZ.$
    Then for a fixed $f \in \mathcal{H}_{\gamma}(\cZ)$, there exists a $g \in L_{2}\left(\mathbb{R}^{m}\right)$ such that $V_{\gamma} g = f$ and
    $\| g \|_{L_{2}\left(\mathbb{R}^{m}\right)} \leq 2\| f \|_{\mathcal{H}_{\gamma}(\cZ)}$.
    
    For $\bm{r} \in \mathbb{N}_0^m$ and $\bm{z} \in \cZ$, we have
    \begin{align}
        |D^{\mathbf{r}} f(\bm{z})| 
        =& \frac{2^{m/2}}{\gamma^{m/2} \pi^{m/4}} 
         \left| D^{\mathbf{r}} \int_{\mathbb{R}^{m}} e^{-2 \gamma^{-2}\|\bm{z}-\bm{y}\|_{2}^{2}} g(\bm{y}) d \bm{y}\right| \nonumber \\
         \leq& \frac{2^{m/2}}{\gamma^{m/2} \pi^{m/4}} 
         \int_{\mathbb{R}^{m}} \left| D^{\mathbf{r}} e^{-2 \gamma^{-2}\|\bm{z}-\bm{y}\|_{2}^{2}} g(\bm{y})\right| d \bm{y}\nonumber \\
         \leq& \frac{2^{m/2}}{\gamma^{m/2} \pi^{m/4}}
         \| g \|_{L_{2}\left(\mathbb{R}^{m}\right)}
         \sqrt{\int_{\mathbb{R}^{m}} \left( D^{\mathbf{r}} e^{-2 \gamma^{-2}\|\bm{z}-\bm{y}\|_{2}^{2}} \right)^2 d \bm{y}}, \label{cor5.4pf}
    \end{align}
    where the last inequality holds by Hölder's inequality.
    
    Now, recall that for $r \in \mathbb{N}_0$, the $r$-th Hermite polynomial is defined by
    \begin{align}
        h_r(t) = (-1)^r e^{t^2} \frac{d^r}{dt^r} e^{-t^2}, t \in \mathbb{R}, \label{Hermite_def}
    \end{align}
    which has the following property
    \begin{align}
        \int_{-\infty}^{\infty} h_{r_1}(t) h_{r_2}(t) e^{-t^2} dt = 2^{r_1} {r_1}! \sqrt{\pi} \delta_{r_1 , r_2}, \label{Hermite_pro}
    \end{align}
    where $\delta_{r_1 , r_2} := \mathbb{I}(r_1 = r_2)$ is the Kronecker symbol.
    By (\ref{Hermite_def}), we have
    \begin{align*}
        \frac{d^{r}}{d t^{r}} e^{-2 \gamma^{-2}(t-s)^{2}}=\left(-\sqrt{2} \gamma^{-1}\right)^{r} e^{-2 \gamma^{-2}(t-s)^{2}} h_{r}\left(\sqrt{2} \gamma^{-1}(t-s)\right)
    \end{align*}
    and hence
    \begin{align}
        \int_{\mathbb{R}}\left|\frac{d^{r}}{d t^{r}} e^{-2 \gamma^{-2}(t-s)^{2}}\right|^{2} d s 
        =&\left(2 \gamma^{-2}\right)^{r} \int_{\mathbb{R}} e^{-4 \gamma^{-2}(t-s)^{2}} h_{r}^{2}\left(\sqrt{2} \gamma^{-1}(t-s)\right) d s \nonumber\\
        =& \left(2 \gamma^{-2}\right)^{r} \int_{\mathbb{R}} e^{-4 \gamma^{-2} s^{2}} h_{r}^{2}\left(\sqrt{2} \gamma^{-1} s\right) d s \nonumber\\
        =&\left(\sqrt{2} \gamma^{-1}\right)^{2 r-1} \int_{\mathbb{R}} e^{-2 s^{2}} h_{r}^{2}(s) d s \nonumber\\
        \leq& \sqrt{\pi} 2^{2 r-1 / 2} r ! \gamma^{1-2 r}, \label{cor5.4pf2}
    \end{align}
    where the last inequality holds by (\ref{Hermite_pro}).
    
    Since $e^{-2 \gamma^{-2}\|\bm{z}-\bm{y}\|_{2}^{2}}=\prod_{i=1}^{m} e^{-2 \gamma^{-2}\left(z_{i}-y_{i}\right)^{2}}$ holds, (\ref{cor5.4pf}) and (\ref{cor5.4pf2}) imply
    \begin{align*}
        |D^{\mathbf{r}} f(\bm{z})| 
        \leq \frac{2^{m/2 + 1}}{\gamma^{m/2} \pi^{m/4}} \| f \|_{\mathcal{H}_{\gamma}(\cZ)}
        \sqrt{\pi^{m/2} 2^{2|\bm{r}|_{1} - m/2} \bm{r}! \gamma^{m-2|\bm{r}|_{1}}},
    \end{align*}
    and thus the main statement holds if we let
    \begin{align*}
        B^{\prime} = \frac{2}{\gamma}\max\left(1, \frac{2^{m/2 + 1} B}{\gamma^{m/2} \pi^{m/4}} \sqrt{\pi^{m/2} 2^{- m/2} \gamma^{m}}   \right).
    \end{align*} \qed


\subsection{About linear prediction models}

For the prediction model $f$ being linear,
the sigmoid IPM can eusure the level of more general DP-fairness.
In fact, the original DP fairness of a prediction model
can be controlled by the sigmoid IPM, which is stated in the following theorem.


\begin{theorem}[Linear classifier] 
\label{thm2}
Suppose $\cF = \{ f :  f(\mathbf{\bz}) = \bm{a}^{\top} \mathbf{z}+b :  \bm{a} \in \mathbb{R}^{m}, b \in \mathbb{R} \}.$
	Then if $d_{\cV_{sig}}(\mathbb{P}_0^h, \mathbb{P}_1^h) < \epsilon$ for a given $\epsilon >0,$ 
	there exists a constant $c>0$ such that
    \begin{equation}
        \begin{split}
            \sup_{f \in \mathcal{F}}\sup_{\tau\in\mathbb{R}} | 
            \mathbbm{E} ( \mathbb{I}(f(\bZ_0)> \tau)) - \mathbbm{E} ( \mathbb{I}(f(\bZ_1)> \tau)) | 
            < c \epsilon
        \end{split}
    \end{equation}
    holds.
\end{theorem}

\begin{proof}[Proof of Theorem \ref{thm2}]
    For a given $f(\bz)=b+\bm{a}^{\top}\bz$ and $\tau$, we have
    \begin{equation*}
        \mathbb{I}(f(\bz)> \tau)=\mathbb{I}(\bm{a}^{\top}\bz> \tau-b).
    \end{equation*}
    Thus, by applying Lemma \ref{lemma1}, the proof is done.
\end{proof}


\newpage

\section{Experimental setup details}
\label{appendix:exp_setup}
\subsection{Dataset pre-processing}\label{appendix:dataset}

For \textit{Adult} and \textit{COMPAS}, we follow the standard pre-processing procedures conducted by \citet{2020alg}. 
As for \textit{Adult}, three variables, education, age, and race, are transformed to categorical variables. 
Specifically, we split the education variable into three categories ($<6$, $6 \le \textup{ and }\le 12$, $< 12$) and we binarize the age variable with a threshold of 70. 
The categorical values for race are repartitioned into two categories, white or non-white. 
And we change all of the categorical variables  to dummy variables. 


And for \textit{COMPAS}, we remove abnormal observations with the pre-specified criterion (days\_b\_screening\_arrest is between -30 and 30, is\_recid is not -1, c\_charge\_degree is not ``O'', and score\_text is not ``N/A'').
Like \textit{Adult}, we replace all the categorical variables to dummy variables.


Regarding \textit{Health}, we pre-process the data as is done in \url{https://github.com/truongkhanhduy95/Heritage-Health-Prize}. 


We summarize the information of three pre-processed datasets in Table \ref{table:dataset}. 

\begin{table}[ht]
	\caption{Descriptions of \textit{Adult}, \textit{COMPAS}, and \textit{Health} after pre-processing.
	}
	\label{table:dataset}
	\vskip 0.15in
	\begin{center}
		\begin{small}
			\begin{tabular}{l|c|c|c|}
				\toprule
				Dataset & Input dimension ($d$) & Representation dimension ($m$) & Sample size (train / val. / test) \\
				\midrule
				\midrule
				\textit{Adult} & 112 & 60 & 24130 / 6032 / 15060 \\
				\textit{COMPAS} & 10 & 8 & 3457 / 864 / 1851 \\
				\textit{Health} & 78 & 40 & 42861 / 14286 / 14287 \\
				\bottomrule
			\end{tabular}
		\end{small}
	\end{center}
	\vskip -0.1in
\end{table}

\subsection{Implementation details}\label{appendix:imp}

The adversarial network is updated two times per each update of the encoder and prediction model (or decoder). 
All the reported results in our paper are achieved by considering various values of $\lambda$. 
We also standardize input vectors for unsupervised LFR because the reconstruction error is well-matched with standardized input vectors rather than raw inputs.
For implementation of other baselines, we refer to the publicly available source codes. 
We re-implement LAFTR with the \texttt{Pytorch} version of LAFTR in \url{https://github.com/VectorInstitute/laftr}. 
And for Fair-Mixup and Fair-Reg, we use the official source codes of Fair-Mixup in \url{https://github.com/chingyaoc/fair-mixup}.

\subsection{Pseudo-code of the sIPM-LFR algorithm}
\label{appendix:alg}

{In this subsection, we provide the sIPM-LFR algorithm  in Algorithm \ref{alg:sipm}.
For unsupervised LFR, we first train the encoder and solve the downstream tasks while fixing the encoder.
The \texttt{Pytorch} implemention of the sIPM-LFR is publicly available in \url{https://github.com/kwkimonline/sIPM-LFR}.
}

\begin{algorithm}[H]
\caption{Algorithm of the sIPM-LFR.}
\label{alg:sipm}
    \begin{algorithmic}[1]
        \REQUIRE: $\texttt{mode} \in \{ \textup{unsup},  \textup{sup} \}:$ the learning setup.
        \REQUIRE $\eta$: parameter of the encoder $h$, $\omega$: parameter of the decoder (if \texttt{mode} == unsup) or prediction function (if \texttt{mode} == sup). 
        \REQUIRE $\psi = [ \theta, \mu ]:$ parameter of the sigmoid discriminator.
        \REQUIRE 
        $\lambda:$ regularization parameter.
        $(\textup{lr},\textup{lr}_{\textup{adv}}):$ two learning rates.
        $(T,T_{\textup{adv}})$: two update numbers.
        $n_{\textup{mb}}:$ mini-batch size.
        \FOR{$i = 1, \cdots, T$}
        \STATE Sample a batch $(\mathbf{x}_{i}, y_{i}, s_{i})_{i=1}^{n_{\textup{mb}}}$ from the training dataset.
        \newline
        \IF{\texttt{mode} == unsup}
        \STATE $\mathcal{L}_{\textup{unsup}}(\eta, \omega) = \frac{1}{n_{\textup{mb}}} \sum_{i=1}^{n_{\textup{mb}}} || \mathbf{x}_{i} - f_{\omega} ( h_{\eta} (\mathbf{x}_{i}) ) ||^{2}$ \hfill \# Compute the reconstruction loss.
        \ELSE
        \STATE $\mathcal{L}_{\textup{sup}}(\eta, \omega) = \frac{1}{n_{\textup{mb}}} \sum_{i=1}^{n_{\textup{mb}}} \textup{cross-entropy}(y_{i}, f_{\omega}(h_{\eta}(\mathbf{x}_{i})) $ \hfill \# Compute the cross-entropy loss.
        \ENDIF
        \newline
        \STATE $\mathcal{L}_{\textup{fair}}(\eta, \psi) = \left| \frac{1}{\sum_{i=1}^{n_{\textup{mb}}} \mathbbm{I}(s_{i} = 0)} \sum_{i : s_{i} = 0} \sigma(\theta^{\top} h_{\eta}(\mathbf{x}_{i}) + \mu) - \frac{1}{\sum_{i=1}^{n_{\textup{mb}}} \mathbbm{I}(s_{i} = 1)} \sum_{i : s_{i} = 1} \sigma(\theta^{\top} h_{\eta}(\mathbf{x}_{i}) + \mu) \right|$ \\ \hfill \# Compute the fair loss.
        \STATE $\mathcal{L}(\eta, \omega, \psi) = \mathcal{L}_{\texttt{mode}}(\eta, \omega) + \lambda \mathcal{L}_{\textup{fair}}(\eta, \psi)$ \hfill \# Compute the total loss. 
        \newline
        \FOR{$t = 1, \cdots, T_{\textup{adv}}$} 
        \STATE $\psi \leftarrow \psi + \textup{lr}_{\textup{adv}}\cdot\nabla_{\psi} \mathcal{L} (\eta, \omega, \psi)$ \hfill \# Update $\psi$ for $T_{\textup{adv}}$ times.
        \ENDFOR
        \STATE $\eta \leftarrow \eta - \textup{lr} \cdot\nabla_{\eta} \mathcal{L} (\eta, \omega, \psi)
        \newline
        \omega \leftarrow \omega - \textup{lr}\cdot\nabla_{\omega} \mathcal{L} (\eta, \omega, \psi)$ \hfill \# Update $\eta$ and $\omega$.
        \ENDFOR
        \newline
        \textbf{Return} $\eta$ and $\omega$ 
    \end{algorithmic}
\end{algorithm}

\section{Fairness measures}\label{appendix:fairmeasures} 
For given a encoder $h$, a prediction model $f$, and a threshold $\tau\in\mathbb{R}$, let $\hat{Y}_\tau=\mathbb{I}(f\circ h(\bm{X},S)>\tau)$ be the predicted label of a random input vector $(\bX,S)$. 
In this paper, we consider four types of DP-fairness measures - 1) original DP, 2) mean DP, 3) strong DP, and 4) variance of DP.
The precise formulas of these fairness measures are provided in Table \ref{table:fairmeasures}.


\begin{table}[ht]
	\caption{Formulas of the four DP-fairness measures. 
   }
	\label{table:fairmeasures}
	\vskip 0.15in
	\begin{center}
		\begin{small}
			\begin{tabular}{l|c}
				\toprule
				Fairness measure & Formula \\
				\midrule
				\midrule
				$\Delta \texttt{DP}$ & $ | \mathbbm{P} ( \hat{Y}_0 = 1 | S = 0 ) - \mathbbm{P} ( \hat{Y}_0 = 1 | S = 1 ) |$ \\
				\midrule
				$\Delta \texttt{MDP}$ & $\left| \mathbbm{E} \left( f\circ h(\mathbf{X},S) | S = 0 \right) - \mathbbm{E} \left( f\circ h(\mathbf{X},S) | S = 1 \right) \right|$ \\
				\midrule
				$\Delta \texttt{SDP}$ & $ \mathbbm{E}_{\tau} ( | \mathbbm{P} ( \hat{Y}_{\tau} = 1 | S = 0 ) - \mathbbm{P} ( \hat{Y}_{\tau} = 1 | S = 1 ) | )$ \\
				\midrule
				$\Delta \texttt{VDP}$ & $\left| \mathbf{Var} \left( f\circ h(\mathbf{X},S) | S = 0 \right) - \mathbf{Var} \left( f\circ h(\mathbf{X},S) | S = 1 \right) \right|$  \\
				\bottomrule
			\end{tabular}
		\end{small}
	\end{center}
	\vskip -0.1in
\end{table}

\newpage

\section{Additional experiments}\label{appendix:exp}

\subsection{Supervised LFR}\label{appendix:sup}

We draw the Pareto-front lines between $\Delta \texttt{SDP}$ and \texttt{acc} in Figure \ref{fig:appensup}. 

\begin{figure}[ht]
\vskip 0.2in
\begin{center}
\centerline{
    \includegraphics[width=0.24\textwidth]{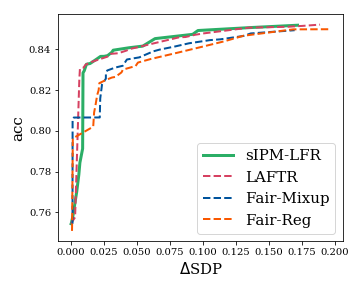}
    \includegraphics[width=0.24\textwidth]{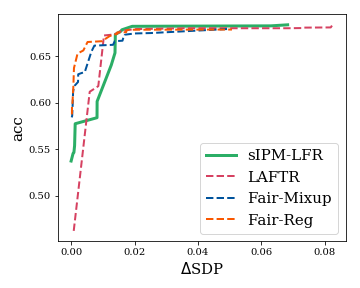}
    \includegraphics[width=0.24\textwidth]{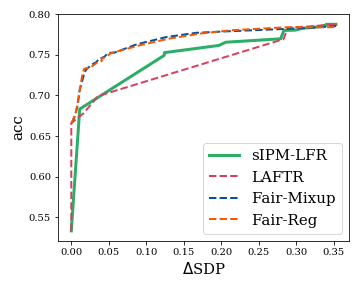}
}
\caption{Supervised LFR: Pareto-front lines between $\Delta \texttt{SDP}$ and \texttt{acc} on the test data
of (left) \textit{Adult}, (center) \textit{COMPAS}, and (right) \textit{Health}.}
\label{fig:appensup}
\end{center}
\vskip -0.2in
\end{figure}

\subsection{Unsupervised LFR}
\label{app}


\paragraph{Additional datasets}

{
Recently, there have been some discussions on the validity of widely-used datasets  for fair AI \cite{https://doi.org/10.48550/arxiv.2108.04884, bao2021its}.
Furthermore, the three tabular datasets analyzed in the main paper have relatively small dimensions.
Under this background, we assess the sIPM-LFR on two additional datasets: \textit{ACSIncome} \textit{Toxicity}.
\begin{itemize}
    \item \textit{ACSIncome} \cite{https://doi.org/10.48550/arxiv.2108.04884}: This dataset is a pre-processed version of \textit{Adult} dataset. Differing from \textit{Adult}, \textit{ACSIncome} only includes individuals above the age of 16, with working hours of at least 1hour/week in the past year, and with income of at least \$100.
    We perform the sIPM-LFR for unsupervised LFR compared to the LAFTR on \textit{ACSIncome} dataset and provide the results in Figure \ref{appendixfig:unsup_acsincome}.
    \item \textit{Toxicity} \footnote{https://www.kaggle.com/c/jigsaw-unintended-bias-in-toxicity-classification}: 
    This dataset is a language dataset (English) containing a large number of Wikipedia comments with ratings of toxicity.
    For input vectors, we use the extracted representations from the encoder of a pre-trained BERT (BERT-base-uncased) \cite{devlin-etal-2019-bert} provided by \texttt{huggingface}\footnote{https://huggingface.co/bert-base-uncased}.
    For class labels, we annotate labels $1$ if the toxicity rating is over $0.5$ and $0$ otherwise. We use the encoder network with two hidden layers and 
    the four classifiers used in Figure \ref{fig:unsup} except the 2-Sigmoid-NN.
    We do not use the 2-Sigmoid-NN due to its gradient vanishing problem.
    We perform the sIPM-LFR for unsupervised LFR compared to the LAFTR on \textit{Toxicity} dataset and provide the Pareto-front lines in Figure \ref{appendixfig:unsup_toxicity}.
\end{itemize}
As can be seen in Figures \ref{appendixfig:unsup_acsincome} and \ref{appendixfig:unsup_toxicity}, we observe similar results to those in Figure \ref{fig:unsup} for the two additional datasets in that the sIPM-LFR is better than the LAFTR in most cases.
}

\begin{figure*}[ht]
\vskip 0.2in
\begin{center}
\centerline{
    \includegraphics[width=0.19\textwidth]{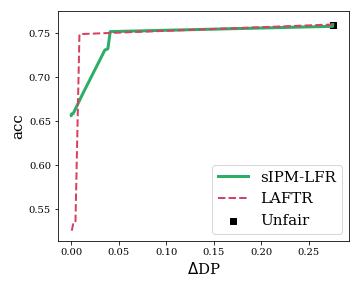}
    \includegraphics[width=0.19\textwidth]{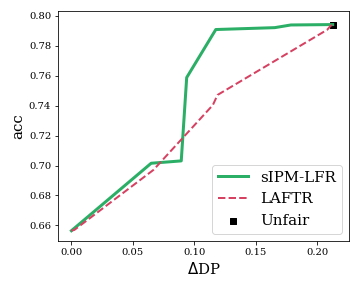}
    \includegraphics[width=0.19\textwidth]{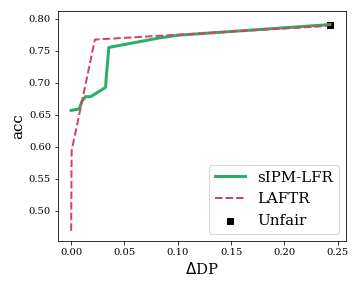}
    \includegraphics[width=0.19\textwidth]{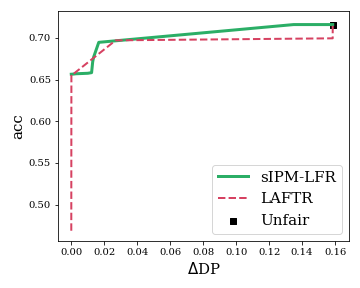}
    \includegraphics[width=0.19\textwidth]{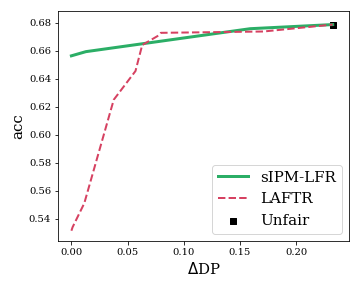}
}
\centerline{
    \includegraphics[width=0.19\textwidth]{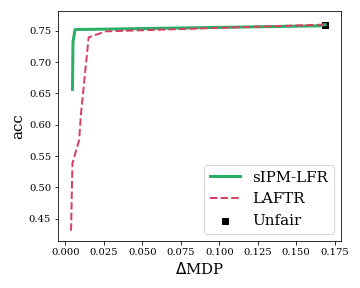}
    \includegraphics[width=0.19\textwidth]{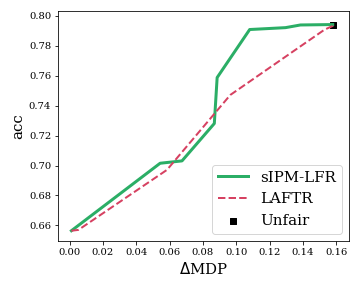}
    \includegraphics[width=0.19\textwidth]{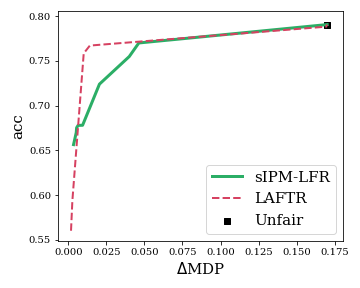}
    \includegraphics[width=0.19\textwidth]{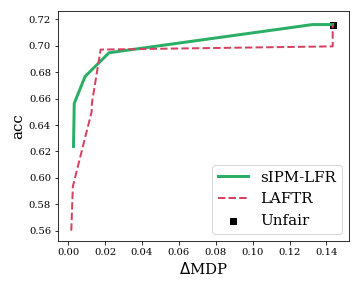}
    \includegraphics[width=0.19\textwidth]{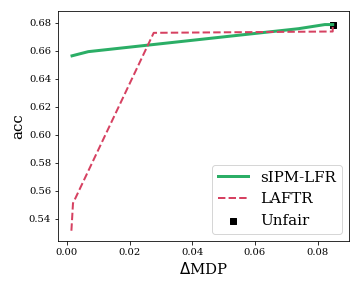}
}
\centerline{
    \includegraphics[width=0.19\textwidth]{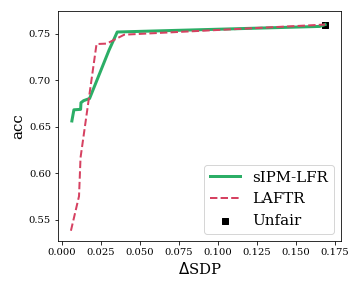}
    \includegraphics[width=0.19\textwidth]{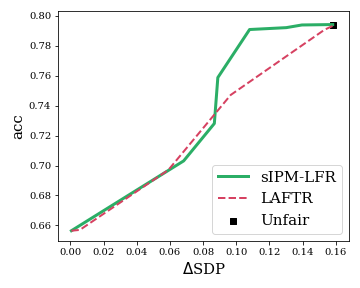}
    \includegraphics[width=0.19\textwidth]{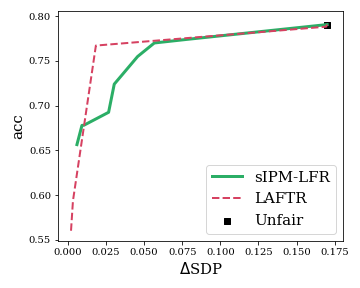}
    \includegraphics[width=0.19\textwidth]{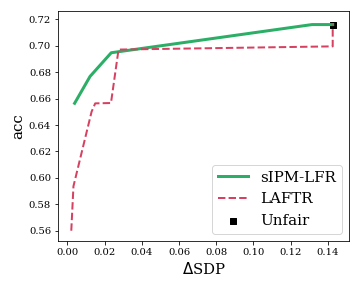}
    \includegraphics[width=0.19\textwidth]{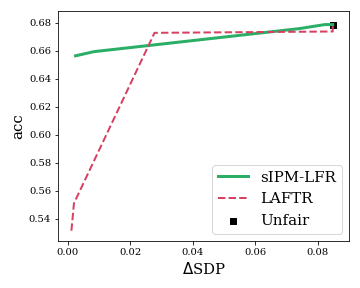}
}
\centerline{
    \includegraphics[width=0.19\textwidth]{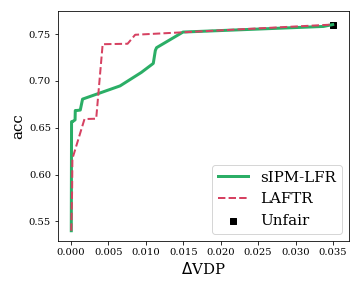}
    \includegraphics[width=0.19\textwidth]{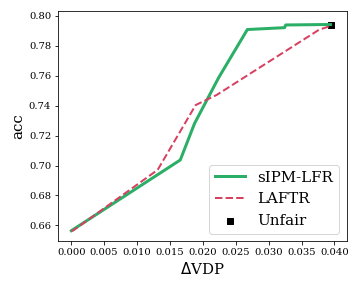}
    \includegraphics[width=0.19\textwidth]{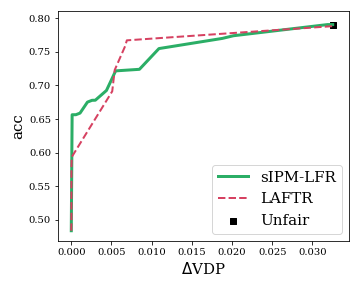}
    \includegraphics[width=0.19\textwidth]{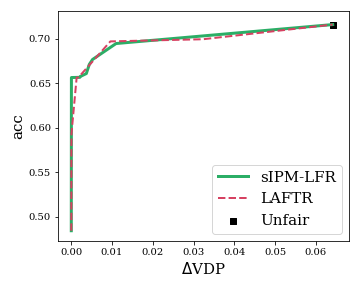}
    \includegraphics[width=0.19\textwidth]{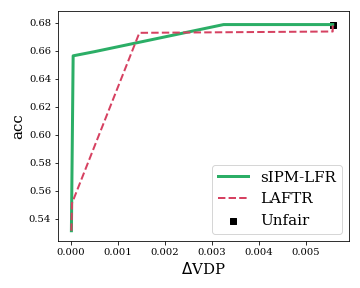}
}
\caption{Unsupervised LFR: Pareto-front lines between $\{ \Delta \texttt{DP}, \Delta \texttt{MDP}, \Delta \texttt{SDP}, \Delta \texttt{VDP} \}$ and \texttt{acc} on the test data of \textit{ACSIncome}.
(left to right) We consider the five prediction models: linear, RBF-SVM, 1-LeakyReLU-NN, 1-Sigmoid-NN, and 2-Sigmoid-NN.}
\label{appendixfig:unsup_acsincome}
\end{center}
\vskip -0.2in
\end{figure*}

\begin{figure*}[ht]
\vskip 0.2in
\begin{center}
\centerline{
    \includegraphics[width=0.19\textwidth]{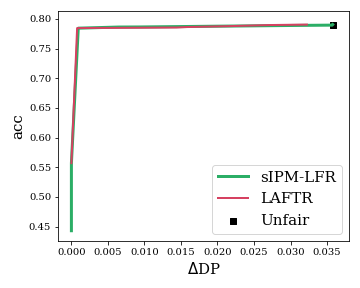}
    \includegraphics[width=0.19\textwidth]{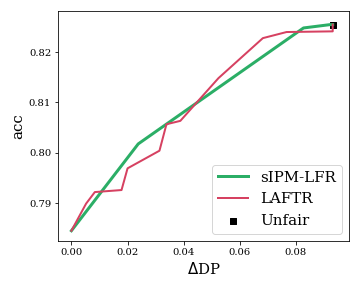}
    \includegraphics[width=0.19\textwidth]{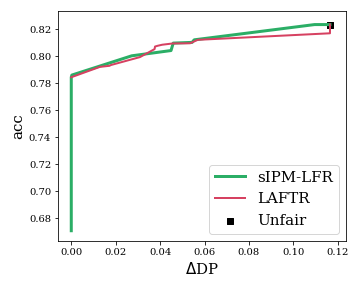}
    \includegraphics[width=0.19\textwidth]{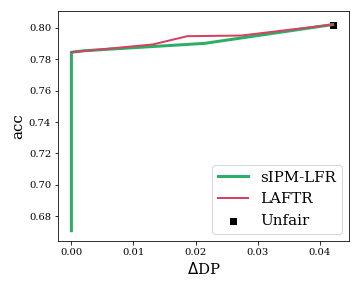}
}
\centerline{
    \includegraphics[width=0.19\textwidth]{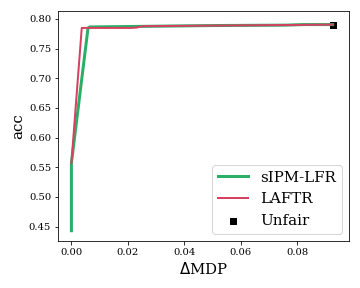}
    \includegraphics[width=0.19\textwidth]{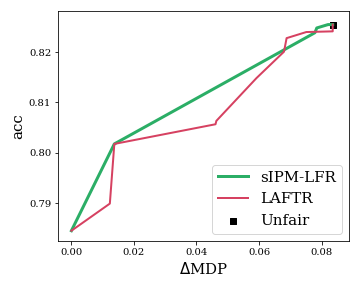}
    \includegraphics[width=0.19\textwidth]{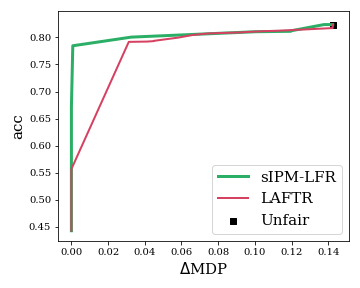}
    \includegraphics[width=0.19\textwidth]{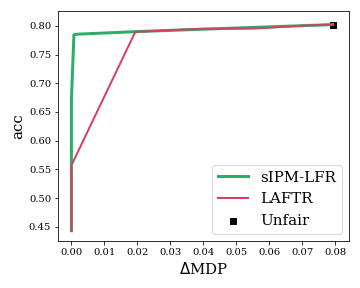}
}
\centerline{
    \includegraphics[width=0.19\textwidth]{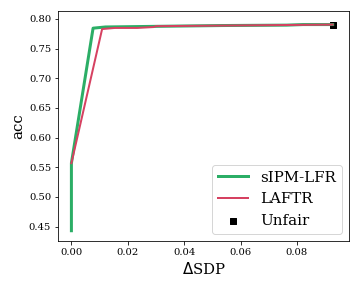}
    \includegraphics[width=0.19\textwidth]{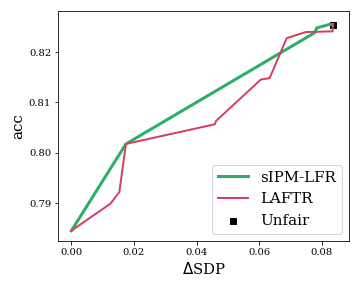}
    \includegraphics[width=0.19\textwidth]{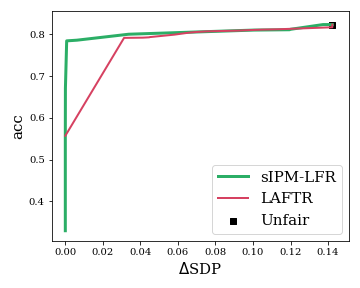}
    \includegraphics[width=0.19\textwidth]{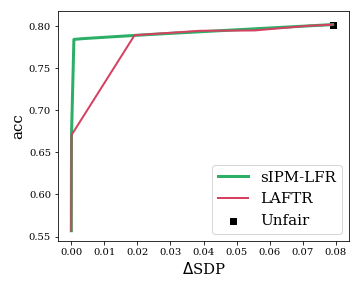}
}
\centerline{
    \includegraphics[width=0.19\textwidth]{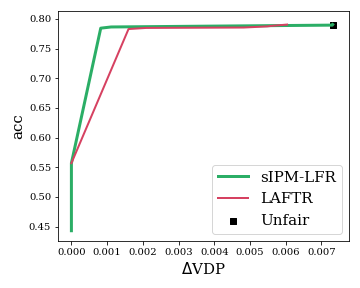}
    \includegraphics[width=0.19\textwidth]{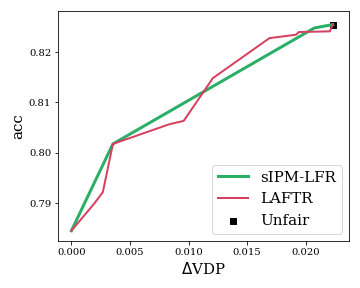}
    \includegraphics[width=0.19\textwidth]{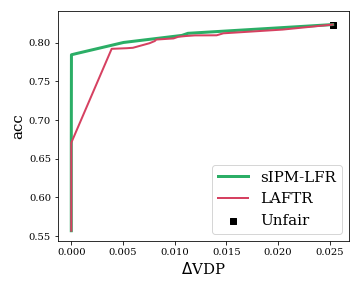}
    \includegraphics[width=0.19\textwidth]{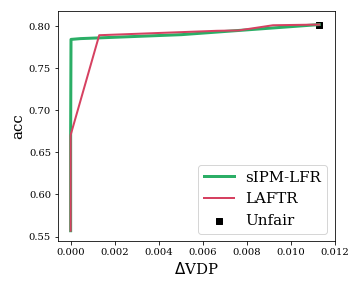}
}
\caption{Unsupervised LFR: Pareto-front lines between $\{ \Delta \texttt{DP}, \Delta \texttt{MDP}, \Delta \texttt{SDP}, \Delta \texttt{VDP} \}$ and \texttt{acc} on the test data of \textit{Toxicity}.
(left to right) We consider the four prediction models: linear, RBF-SVM, 1-LeakyReLU-NN, and 1-Sigmoid-NN.}
\label{appendixfig:unsup_toxicity}
\end{center}
\vskip -0.2in
\end{figure*}

\paragraph{Trade-offs between $\{ \Delta \texttt{MDP}, \Delta \texttt{SDP}, \Delta \texttt{VDP} \}$ and \texttt{acc}.}

We provide the Pareto-front lines (Figure \ref{appendixfig:unsup_adult}, \ref{appendixfig:unsup_compas}, and \ref{appendixfig:unsup_health}) for more measures of fairness: $\Delta \texttt{MDP}, \Delta \texttt{SDP}, \Delta \texttt{VDP}$. We confirm that the results are similar to Figure \ref{fig:unsup}.

\paragraph{Visualization of learned representations}
Figure \ref{fig:tsne} visualizes the representation distributions for each sensitive group derived by the sIPM-LFR
with various regularization parameters. 
We can observe that the larger $\lambda$ becomes, the more fair the encoded representation is. That is, we can control the
fairness of representation (and thus fairness of the final
prediction model) nicely by choosing $\lambda$ accordingly.


\paragraph{Simulation for \textit{Adult} with artificial $Y$}

We verify our method's superiority on unsupervised learning by an additional downstream classification task with artificial labels. 
We consider \textit{Adult} and the artificial labels are generated as follows. 
We first train the encoder $h$ and decoder $f_D$ only with the reconstruction loss. 
And we draw an $m$-dimensional random vector $\gamma$ from $\mathcal{N}(0_{m}, 2 I_{m})$ and fix it until the label generation process ends. 
Then, for each input sample $(\bx,s)$, we sample a random vector $\epsilon\sim\mathcal{N}(0_{m}, 2 I_{m})$ and generate its artificial label as $\mathbb{I}(\gamma^{\top}h(\bx,s)+\epsilon)$. 
We analyze \textit{Adult} with the artificial labels by comparing our method and the LAFTR, whose results are depicted in Figure \ref{fig:art1}. 
We utilize the linear prediction model and consider three DP-fairness measures, $\Delta \texttt{DP}, \Delta \texttt{MDP}, \Delta \texttt{SDP}$. 
Figure \ref{fig:art1} shows that our method achieves consistently better trade-off results between the accuracy and DP measures.


\begin{figure*}[ht]
\vskip 0.2in
\begin{center}
\centerline{
    \includegraphics[width=0.19\textwidth]{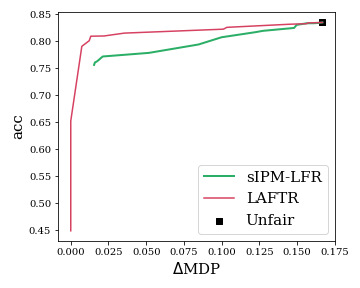}
    \includegraphics[width=0.19\textwidth]{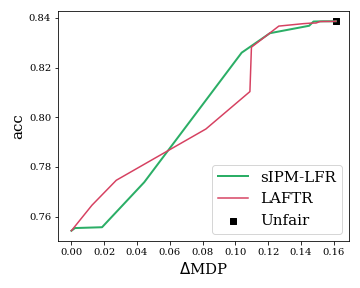}
    \includegraphics[width=0.19\textwidth]{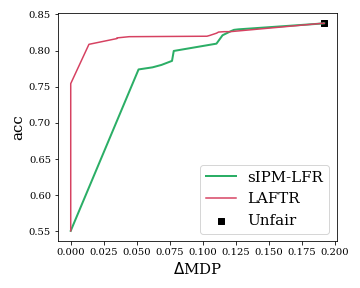}
    \includegraphics[width=0.19\textwidth]{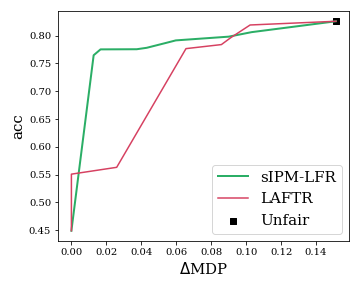}
    \includegraphics[width=0.19\textwidth]{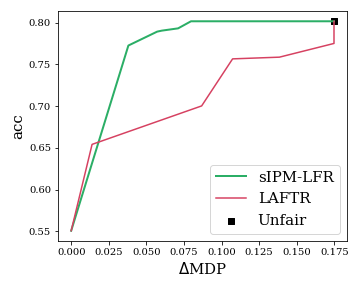}
}
\centerline{
    \includegraphics[width=0.19\textwidth]{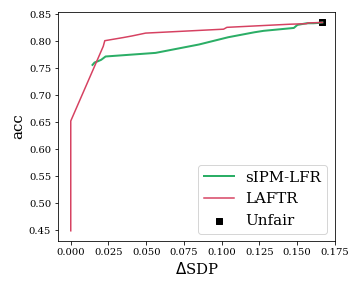}
    \includegraphics[width=0.19\textwidth]{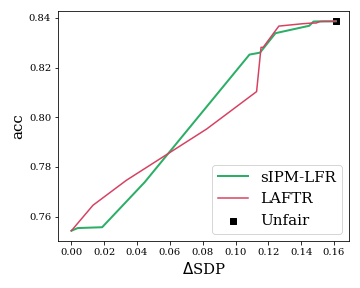}
    \includegraphics[width=0.19\textwidth]{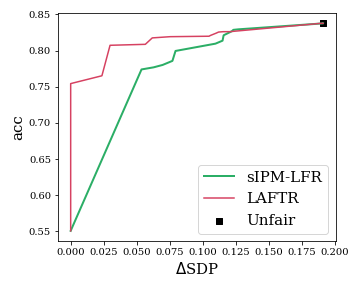}
    \includegraphics[width=0.19\textwidth]{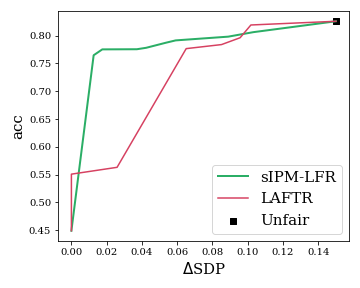}
    \includegraphics[width=0.19\textwidth]{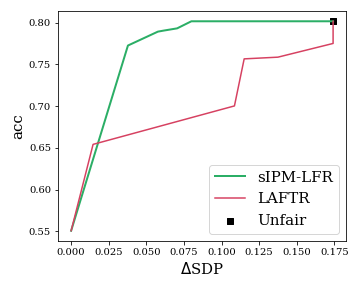}
}
\centerline{
    \includegraphics[width=0.19\textwidth]{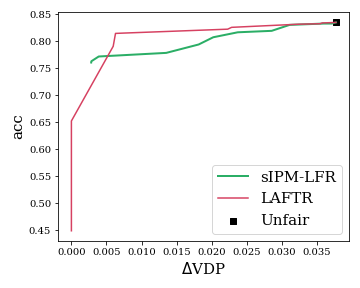}
    \includegraphics[width=0.19\textwidth]{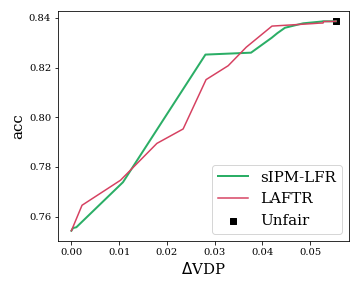}
    \includegraphics[width=0.19\textwidth]{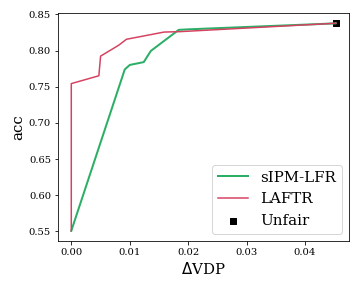}
    \includegraphics[width=0.19\textwidth]{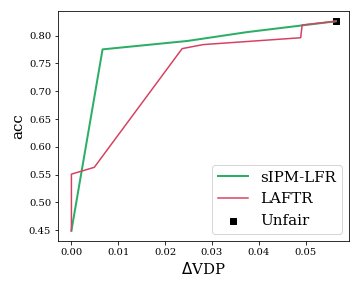}
    \includegraphics[width=0.19\textwidth]{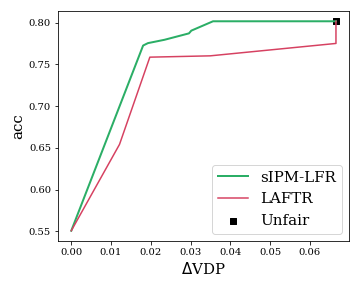}
}
\caption{Unsupervised LFR: Pareto-front lines between $\{ \Delta \texttt{MDP}, \Delta \texttt{SDP}, \Delta \texttt{VDP} \}$ and \texttt{acc} on the test data of \textit{Adult}.
(left to right) We consider the five prediction models: linear, RBF-SVM, 1-LeakyReLU-NN, 1-Sigmoid-NN, and 2-Sigmoid-NN.}
\label{appendixfig:unsup_adult}
\end{center}
\vskip -0.2in
\end{figure*}

\begin{figure*}[ht]
\vskip 0.2in
\begin{center}
\centerline{
    \includegraphics[width=0.19\textwidth]{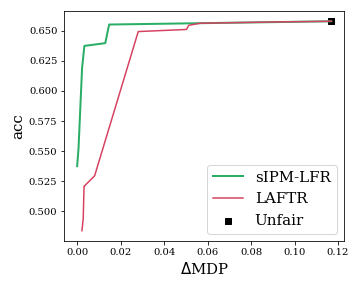}
    \includegraphics[width=0.19\textwidth]{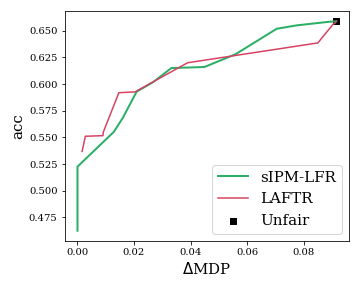}
    \includegraphics[width=0.19\textwidth]{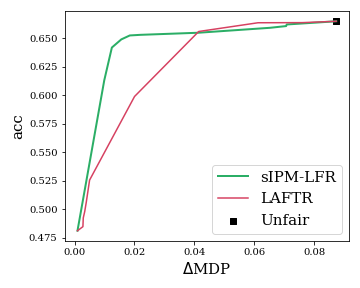}
    \includegraphics[width=0.19\textwidth]{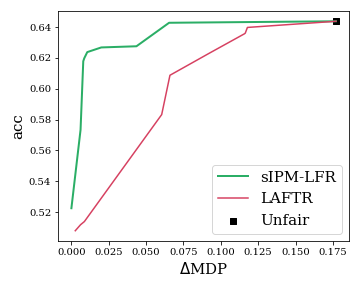}
    \includegraphics[width=0.19\textwidth]{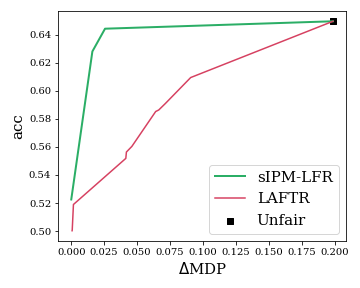}
}
\centerline{
    \includegraphics[width=0.19\textwidth]{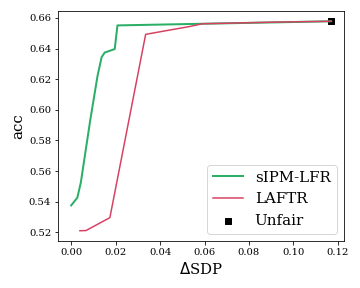}
    \includegraphics[width=0.19\textwidth]{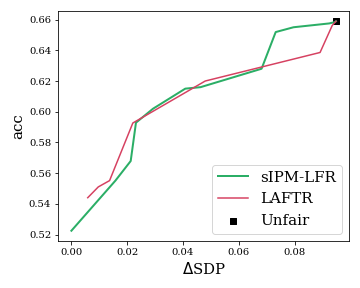}
    \includegraphics[width=0.19\textwidth]{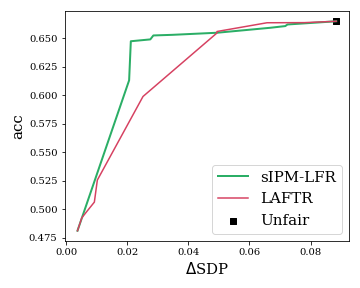}
    \includegraphics[width=0.19\textwidth]{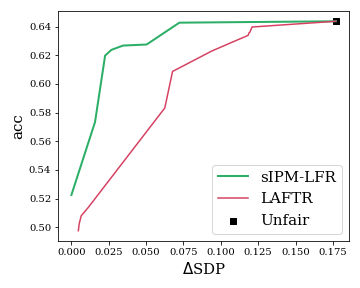}
    \includegraphics[width=0.19\textwidth]{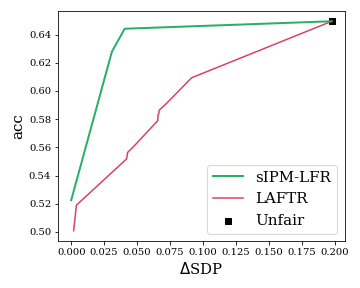}
}
\centerline{
    \includegraphics[width=0.19\textwidth]{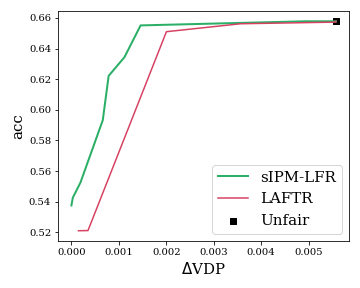}
    \includegraphics[width=0.19\textwidth]{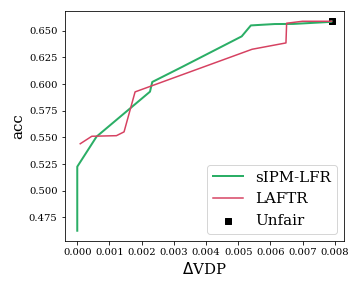}
    \includegraphics[width=0.19\textwidth]{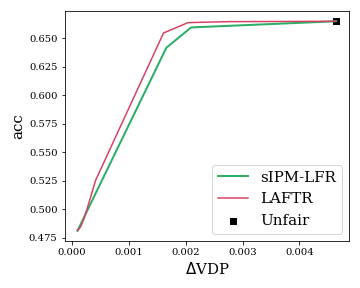}
    \includegraphics[width=0.19\textwidth]{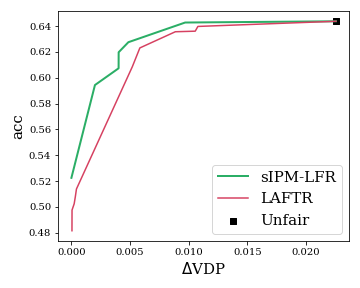}
    \includegraphics[width=0.19\textwidth]{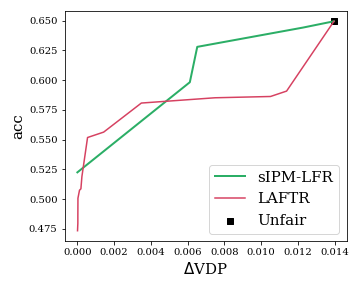}
}
\caption{Unsupervised LFR: Pareto-front lines between $\{ \Delta \texttt{MDP}, \Delta \texttt{SDP}, \Delta \texttt{VDP} \}$ and \texttt{acc} on the test data of \textit{COMPAS}.
(left to right) We consider the five prediction models: linear, RBF-SVM, 1-LeakyReLU-NN, 1-Sigmoid-NN, and 2-Sigmoid-NN.}
\label{appendixfig:unsup_compas}
\end{center}
\vskip -0.2in
\end{figure*}

\begin{figure*}[ht]
\vskip 0.2in
\begin{center}
\centerline{
    \includegraphics[width=0.19\textwidth]{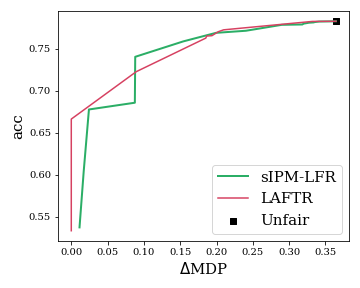}
    \includegraphics[width=0.19\textwidth]{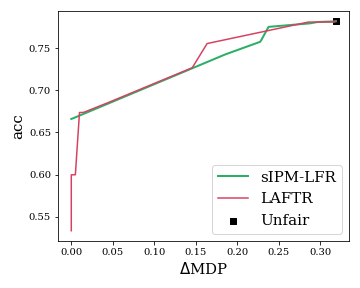}
    \includegraphics[width=0.19\textwidth]{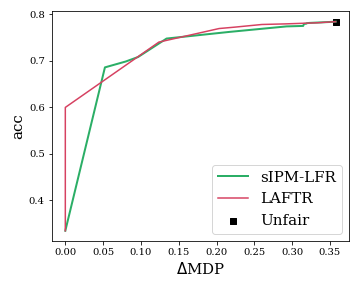}
    \includegraphics[width=0.19\textwidth]{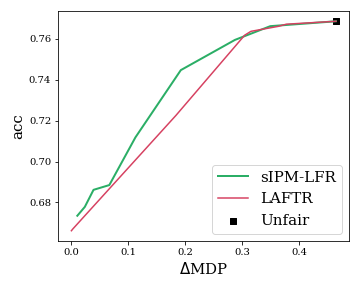}
    \includegraphics[width=0.19\textwidth]{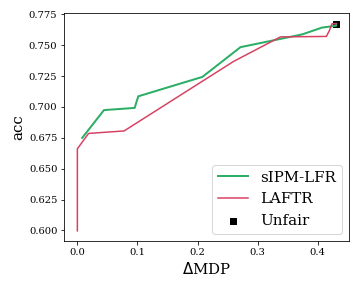}
}
\centerline{
    \includegraphics[width=0.19\textwidth]{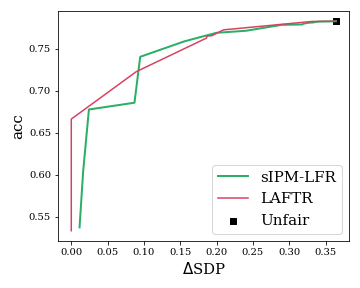}
    \includegraphics[width=0.19\textwidth]{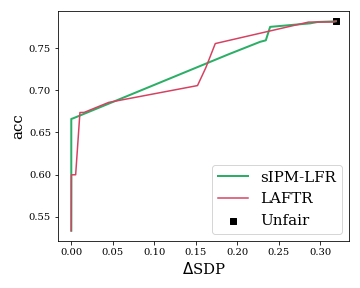}
    \includegraphics[width=0.19\textwidth]{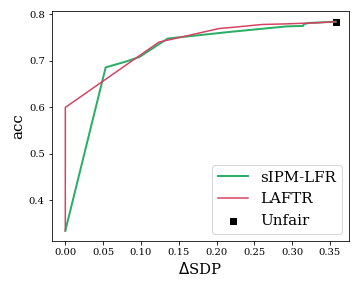}
    \includegraphics[width=0.19\textwidth]{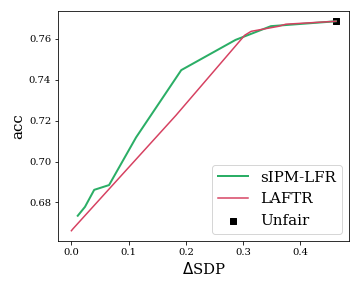}
    \includegraphics[width=0.19\textwidth]{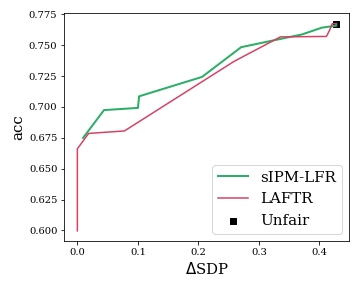}
}
\centerline{
    \includegraphics[width=0.19\textwidth]{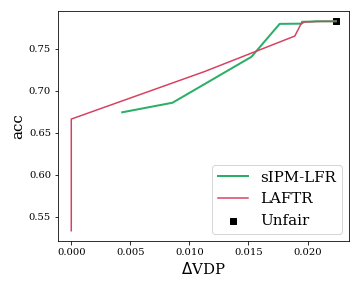}
    \includegraphics[width=0.19\textwidth]{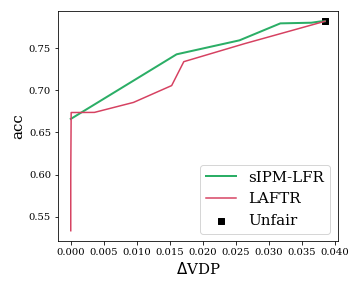}
    \includegraphics[width=0.19\textwidth]{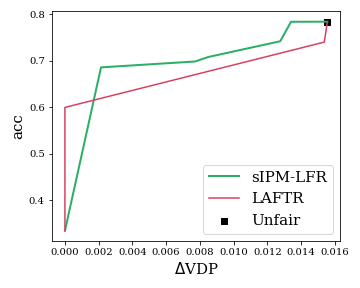}
    \includegraphics[width=0.19\textwidth]{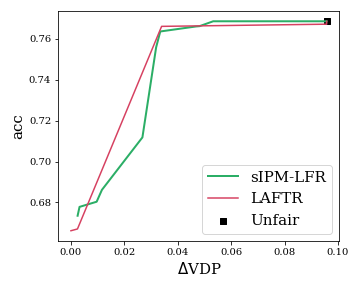}
    \includegraphics[width=0.19\textwidth]{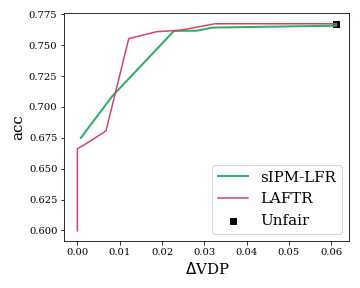}
}
\caption{Unsupervised LFR: Pareto-front lines between $\{ \Delta \texttt{MDP}, \Delta \texttt{SDP}, \Delta \texttt{VDP} \}$ and \texttt{acc} on the test data of \textit{Health}.
(left to right) We consider the five prediction models: linear, RBF-SVM, 1-LeakyReLU-NN, 1-Sigmoid-NN, and 2-Sigmoid-NN.}
\label{appendixfig:unsup_health}
\end{center}
\vskip -0.2in
\end{figure*}

\begin{figure}[ht]
\vskip 0.2in
\begin{center}
\centerline{
    \includegraphics[width=0.24\textwidth]{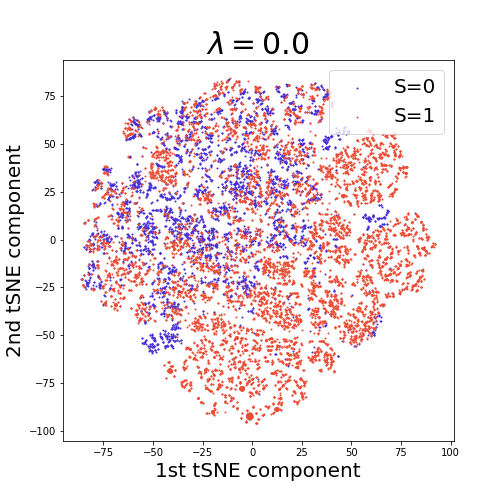}
    \includegraphics[width=0.24\textwidth]{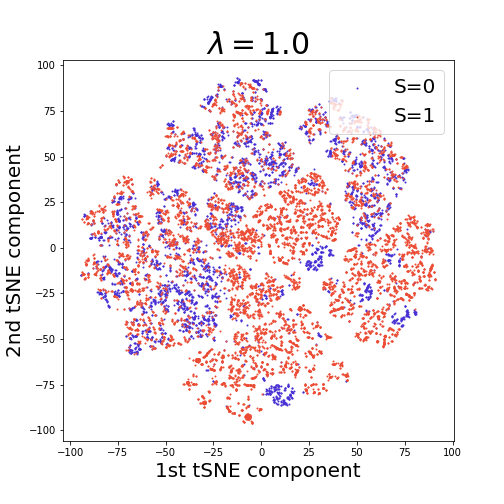}
    \includegraphics[width=0.24\textwidth]{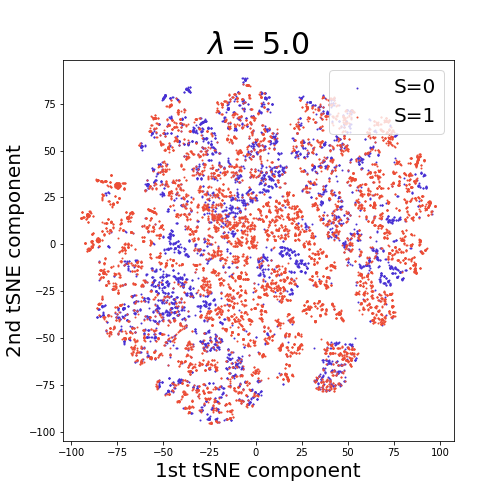}
    \includegraphics[width=0.24\textwidth]{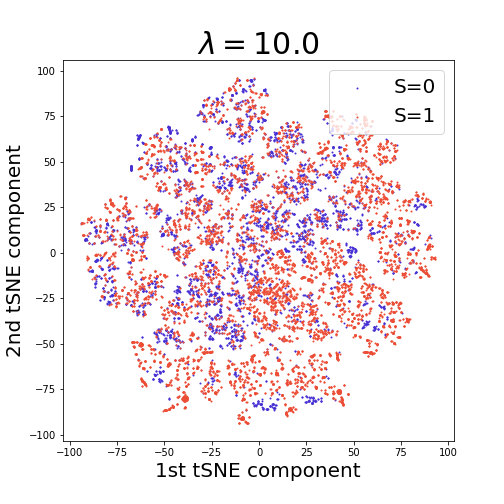}
}
\centerline{
    \includegraphics[width=0.24\textwidth]{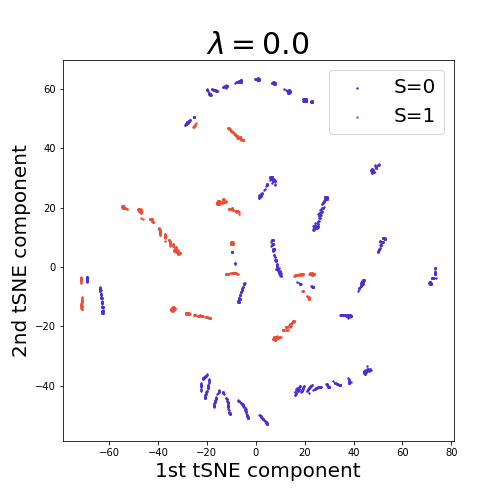}
    \includegraphics[width=0.24\textwidth]{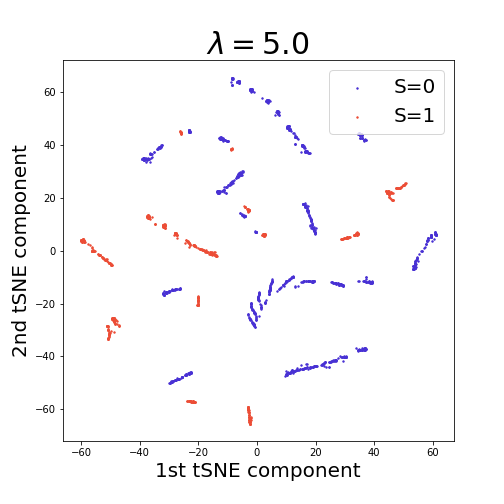}
    \includegraphics[width=0.24\textwidth]{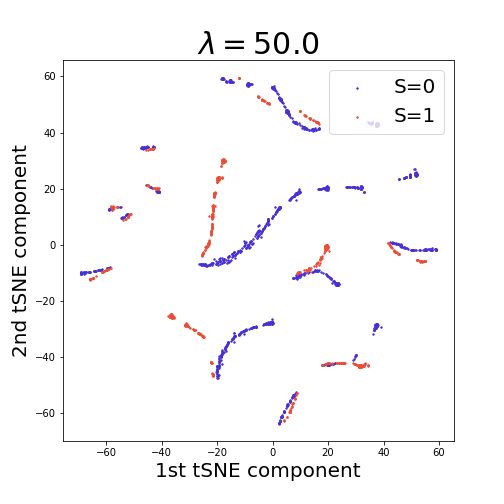}
    \includegraphics[width=0.24\textwidth]{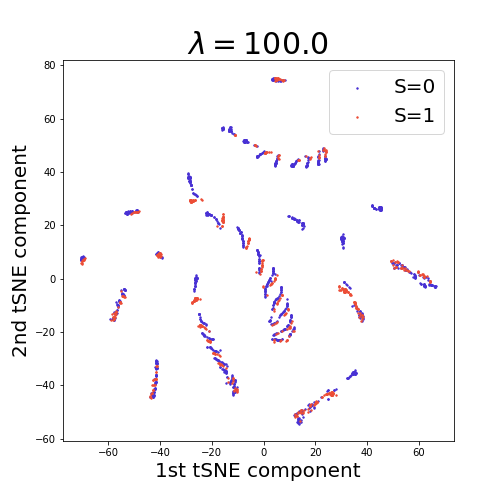}
}
\caption{Unsupervised LFR: tSNE visualization of 
the learned fair representation for
(upper) \textit{Adult} and (lower) \textit{COMPAS} with various values of $\lambda$.  
}
\label{fig:tsne}
\end{center}
\vskip -0.2in
\end{figure}

\begin{figure}[ht]
\vskip 0.2in
\begin{center}
\centerline{
    \includegraphics[width=0.24\textwidth]{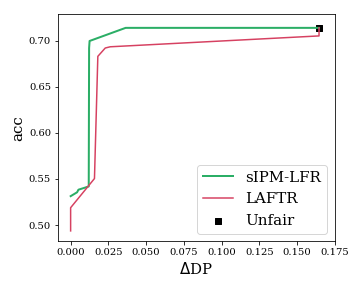}
    \includegraphics[width=0.24\textwidth]{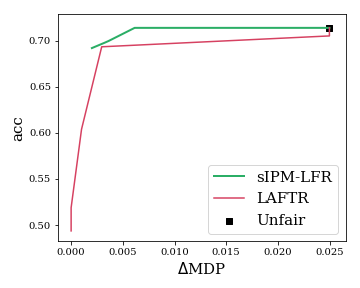}
    \includegraphics[width=0.24\textwidth]{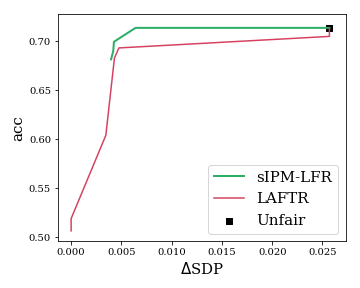}
}
\caption{Unsupervised LFR: Pareto-front lines of $\{ \Delta \texttt{DP}, \Delta \texttt{MDP}, \Delta \texttt{SDP} \}$ (x-axis) vs. \texttt{acc} (y-axis) on \textit{Adult} with the artificial label.}
\label{fig:art1}
\end{center}
\vskip -0.2in
\end{figure}

\section{Ablation studies}
\label{appendix:abl}

This section provides additional ablation experiments that are not included in the main manuscript.

\paragraph{Computation time}
We conduct learning-time comparisons for the sIPM-LFR and LAFTR.  
As can be seen in Table \ref{table:comptime}, sIPM-LFR requires about 20\% less computation times compared to the LAFTR. 

\paragraph{Varying the dimension of the representation}
We analyze the effect of varying the representation's dimension $m$. 
For each dataset, we consider two values of $m$, and compare their performances with the Pareto-front lines. 
As shown in Figure \ref{fig:nodesdims}, our method is more insensitive to the selection of $m$ compared to the LAFTR. 

{
\paragraph{sIPM-LFR vs. MMD-LFR}
We compare the sIPM-LFR to the FVAE \cite{https://doi.org/10.48550/arxiv.1511.00830} which is one of the MMD-based LFR methods.
Theoretically, the MMD regularization in the FVAE is also a kind of IPM that utilizes a unit ball in an RKHS as $\mathcal{V}$ (the class of discriminators).
We can easily show that Theorem \ref{thm3} and Proposition \ref{prop:rbf} imply that the MMD with the Gaussian kernel is upper bounded by the sIPM. 
That is, by controlling the parametric IPM, we expect that the MMD will be also reduced.
}

{
An obvious practical advantage of the sIPM-LFR over the FVAE would be computational simplicity. 
We conduct an experiment to compare the stability and performance between the sIPM-LFR and FVAE. 
Figure \ref{appendixfig:unsup_mmd} depicts the scatter points with standard errors for $\Delta \texttt{DP}$ and \texttt{acc} for 1-Sigmoid-NN on \textit{Adult} dataset.
We can check that the sIPM-LFR is more stable as well as superior compared to the FVAE, which again validates the superiority of our method. 
}





\begin{table}[ht]
	\caption{Training time comparisons between the sIPM-LFR and LAFTR. We report each method's mean and standard values with five random implementations.
}
	\label{table:comptime}
	\vskip 0.15in
	\begin{center}
		\begin{small}
			\begin{tabular}{c||c|c}
				\toprule
				Dataset 
				& Method & Computation Time (s.e.) \\
				\midrule
				\midrule
				{\textit{Adult} } & sIPM-LFR \checkmark & \textbf{100.00\%} (0.80\%)  \\
				& LAFTR & 117.48\% (0.33\%) \\
				\midrule
				{\textit{COMPAS} }& sIPM-LFR \checkmark & \textbf{100.00\%} (3.23\%)  \\
				& LAFTR & 121.13\% (1.77\%) \\
				\midrule
				{\textit{Health} }& sIPM-LFR \checkmark & \textbf{100.00\%} (0.91\%)  \\
				& LAFTR & 117.81\% (0.53\%) \\
				\bottomrule
			\end{tabular}
		\end{small}
	\end{center}
	\vskip -0.1in
\end{table}

\begin{figure}[ht]
\vskip 0.2in
\begin{center}
\centerline{
    \includegraphics[width=0.24\textwidth]{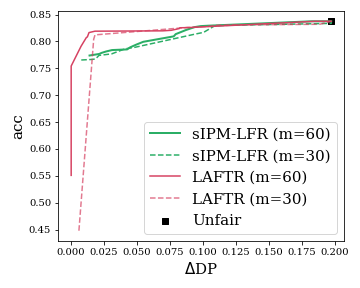}
    \includegraphics[width=0.24\textwidth]{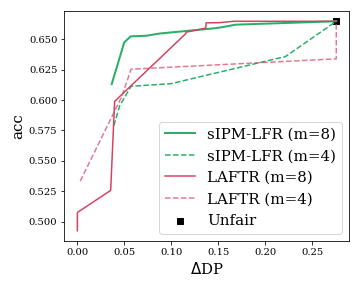}
    \includegraphics[width=0.24\textwidth]{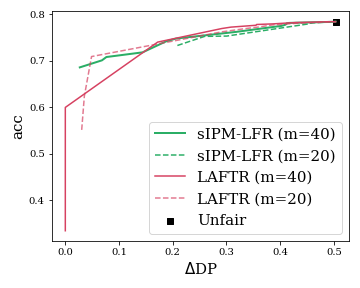}
}
\caption{Unsupervised LFR: Pareto-front lines between $\Delta \texttt{DP}$ (x-axis) and \texttt{acc} (y-axis) with different values of the representation dimension. We analyze three datasets: (left) \textit{Adult}, (center) \textit{COMPAS}, and (right) \textit{Health}. We utilize the 1-LeakyReLU-NN as the prediction model. 
}
\label{fig:nodesdims}
\end{center}
\vskip -0.2in
\end{figure}

\begin{figure*}[ht]
\vskip 0.2in
\begin{center}
\centerline{
    \includegraphics[width=0.35\textwidth]{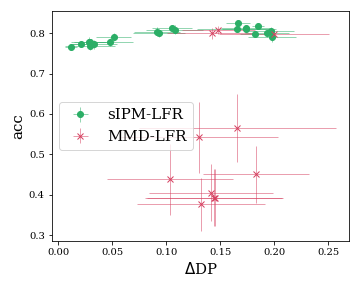}
}
\caption{Scatter plot with standard error bar of $\Delta\texttt{DP}$ and \texttt{acc} with various $\lambda$. 
Each horizontal and vertical bars present the standard errors for $\Delta\texttt{DP}$ and \texttt{acc}, respectively.
All results are from \textit{Adult} test dataset.}
\label{appendixfig:unsup_mmd}
\end{center}
\vskip -0.2in
\end{figure*}



\end{document}

%% file: math_commands.tex
\usepackage{amsmath,amsfonts,bm}

\def\eqref#1{equation~\ref{#1}}

\def\1{\bm{1}}

\DeclareMathAlphabet{\mathsfit}{\encodingdefault}{\sfdefault}{m}{sl}
\SetMathAlphabet{\mathsfit}{bold}{\encodingdefault}{\sfdefault}{bx}{n}

\newcommand{\E}{\mathbb{E}}



%% file: sIPM_LFR_cameraready_final.bbl
\begin{thebibliography}{52}
\providecommand{\natexlab}[1]{#1}
\providecommand{\url}[1]{\texttt{#1}}
\expandafter\ifx\csname urlstyle\endcsname\relax
  \providecommand{\doi}[1]{doi: #1}\else
  \providecommand{\doi}{doi: \begingroup \urlstyle{rm}\Url}\fi

\bibitem[Agarwal et~al.(2018)Agarwal, Beygelzimer, Dudik, Langford, and
  Wallach]{pmlr-v80-agarwal18a}
Agarwal, A., Beygelzimer, A., Dudik, M., Langford, J., and Wallach, H.
\newblock A reductions approach to fair classification.
\newblock In Dy, J. and Krause, A. (eds.), \emph{Proceedings of the 35th
  International Conference on Machine Learning}, volume~80 of \emph{Proceedings
  of Machine Learning Research}, pp.\  60--69. PMLR, 10--15 Jul 2018.
\newblock URL \url{https://proceedings.mlr.press/v80/agarwal18a.html}.

\bibitem[Angwin et~al.(2016)Angwin, Larson, Mattu, and
  Kirchner]{angwin2016machine}
Angwin, J., Larson, J., Mattu, S., and Kirchner, L.
\newblock Machine bias.
\newblock \emph{ProPublica, May}, 23:\penalty0 2016, 2016.

\bibitem[Ansari et~al.(2020)Ansari, Scarlett, and
  Soh]{ansari2020characteristic}
Ansari, A.~F., Scarlett, J., and Soh, H.
\newblock A characteristic function approach to deep implicit generative
  modeling, 2020.

\bibitem[Arjovsky et~al.(2017)Arjovsky, Chintala, and
  Bottou]{10.5555/3305381.3305404}
Arjovsky, M., Chintala, S., and Bottou, L.
\newblock Wasserstein generative adversarial networks.
\newblock In \emph{Proceedings of the 34th International Conference on Machine
  Learning - Volume 70}, ICML'17, pp.\  214–223. JMLR.org, 2017.

\bibitem[Bao et~al.(2021)Bao, Zhou, Zottola, Brubach, Desmarais, Horowitz, Lum,
  and Venkatasubramanian]{bao2021its}
Bao, M., Zhou, A., Zottola, S.~A., Brubach, B., Desmarais, S., Horowitz, A.~S.,
  Lum, K., and Venkatasubramanian, S.
\newblock It's {COMPAS}licated: The messy relationship between {RAI} datasets
  and algorithmic fairness benchmarks.
\newblock In \emph{Thirty-fifth Conference on Neural Information Processing
  Systems Datasets and Benchmarks Track (Round 1)}, 2021.
\newblock URL \url{https://openreview.net/forum?id=qeM58whnpXM}.

\bibitem[Barocas \& Selbst(2016)Barocas and Selbst]{barocas2016big}
Barocas, S. and Selbst, A.~D.
\newblock Big data's disparate impact.
\newblock \emph{Calif. L. Rev.}, 104:\penalty0 671, 2016.

\bibitem[Barron(1993)]{barron1993universal}
Barron, A.~R.
\newblock Universal approximation bounds for superpositions of a sigmoidal
  function.
\newblock \emph{IEEE Transactions on Information theory}, 39\penalty0
  (3):\penalty0 930--945, 1993.

\bibitem[Calders et~al.(2009)Calders, Kamiran, and
  Pechenizkiy]{calders2009building}
Calders, T., Kamiran, F., and Pechenizkiy, M.
\newblock Building classifiers with independency constraints.
\newblock In \emph{2009 IEEE International Conference on Data Mining
  Workshops}, pp.\  13--18. IEEE, 2009.

\bibitem[Chiappa(2019)]{Chiappa_2019}
Chiappa, S.
\newblock Path-specific counterfactual fairness.
\newblock \emph{Proceedings of the AAAI Conference on Artificial Intelligence},
  33\penalty0 (01):\penalty0 7801--7808, Jul. 2019.
\newblock \doi{10.1609/aaai.v33i01.33017801}.
\newblock URL \url{https://ojs.aaai.org/index.php/AAAI/article/view/4777}.

\bibitem[Chuang \& Mroueh(2021)Chuang and Mroueh]{chuang2021fair}
Chuang, C.-Y. and Mroueh, Y.
\newblock Fair mixup: Fairness via interpolation.
\newblock 2021.

\bibitem[Cortes \& Vapnik(1995)Cortes and Vapnik]{10.1023/A:1022627411411}
Cortes, C. and Vapnik, V.
\newblock Support-vector networks.
\newblock \emph{Mach. Learn.}, 20\penalty0 (3):\penalty0 273–297, sep 1995.
\newblock ISSN 0885-6125.
\newblock \doi{10.1023/A:1022627411411}.
\newblock URL \url{https://doi.org/10.1023/A:1022627411411}.

\bibitem[Creager et~al.(2019)Creager, Madras, Jacobsen, Weis, Swersky, Pitassi,
  and Zemel]{pmlr-v97-creager19a}
Creager, E., Madras, D., Jacobsen, J.-H., Weis, M., Swersky, K., Pitassi, T.,
  and Zemel, R.
\newblock Flexibly fair representation learning by disentanglement.
\newblock In Chaudhuri, K. and Salakhutdinov, R. (eds.), \emph{Proceedings of
  the 36th International Conference on Machine Learning}, volume~97 of
  \emph{Proceedings of Machine Learning Research}, pp.\  1436--1445. PMLR,
  09--15 Jun 2019.
\newblock URL \url{https://proceedings.mlr.press/v97/creager19a.html}.

\bibitem[Devlin et~al.(2019)Devlin, Chang, Lee, and
  Toutanova]{devlin-etal-2019-bert}
Devlin, J., Chang, M.-W., Lee, K., and Toutanova, K.
\newblock {BERT}: Pre-training of deep bidirectional transformers for language
  understanding.
\newblock In \emph{Proceedings of the 2019 Conference of the North {A}merican
  Chapter of the Association for Computational Linguistics: Human Language
  Technologies, Volume 1 (Long and Short Papers)}, pp.\  4171--4186,
  Minneapolis, Minnesota, June 2019. Association for Computational Linguistics.
\newblock \doi{10.18653/v1/N19-1423}.
\newblock URL \url{https://aclanthology.org/N19-1423}.

\bibitem[Ding et~al.(2021)Ding, Hardt, Miller, and
  Schmidt]{https://doi.org/10.48550/arxiv.2108.04884}
Ding, F., Hardt, M., Miller, J., and Schmidt, L.
\newblock Retiring adult: New datasets for fair machine learning, 2021.
\newblock URL \url{https://arxiv.org/abs/2108.04884}.

\bibitem[Donini et~al.(2018)Donini, Oneto, Ben-David, Shawe-Taylor, and
  Pontil]{donini2018empirical}
Donini, M., Oneto, L., Ben-David, S., Shawe-Taylor, J.~S., and Pontil, M.
\newblock Empirical risk minimization under fairness constraints.
\newblock In \emph{Advances in Neural Information Processing Systems}, pp.\
  2791--2801, 2018.

\bibitem[Dua \& Graff(2017)Dua and Graff]{Dua:2019}
Dua, D. and Graff, C.
\newblock {UCI} machine learning repository, 2017.
\newblock URL \url{http://archive.ics.uci.edu/ml}.

\bibitem[Dwork et~al.(2012)Dwork, Hardt, Pitassi, Reingold, and Zemel]{dwork}
Dwork, C., Hardt, M., Pitassi, T., Reingold, O., and Zemel, R.
\newblock Fairness through awareness.
\newblock ITCS '12, pp.\  214–226, New York, NY, USA, 2012. Association for
  Computing Machinery.
\newblock ISBN 9781450311151.
\newblock \doi{10.1145/2090236.2090255}.
\newblock URL \url{https://doi.org/10.1145/2090236.2090255}.

\bibitem[Edwards \& Storkey(2016)Edwards and
  Storkey]{9aa5ba8a091248d597ff7cf0173da151}
Edwards, H. and Storkey, A.
\newblock Censoring representations with an adversary.
\newblock In \emph{International Conference in Learning Representations
  (ICLR2016)}, pp.\  1--14, May 2016.
\newblock URL
  \url{https://iclr.cc/archive/www/doku.php%3Fid=iclr2016:main.html}.
\newblock 4th International Conference on Learning Representations, ICLR 2016 ;
  Conference date: 02-05-2016 Through 04-05-2016.

\bibitem[Feldman et~al.(2015)Feldman, Friedler, Moeller, Scheidegger, and
  Venkatasubramanian]{feldman2015certifying}
Feldman, M., Friedler, S.~A., Moeller, J., Scheidegger, C., and
  Venkatasubramanian, S.
\newblock Certifying and removing disparate impact.
\newblock In \emph{proceedings of the 21th ACM SIGKDD international conference
  on knowledge discovery and data mining}, pp.\  259--268, 2015.

\bibitem[Garg et~al.(2019)Garg, Perot, Limtiaco, Taly, Chi, and
  Beutel]{10.1145/3306618.3317950}
Garg, S., Perot, V., Limtiaco, N., Taly, A., Chi, E.~H., and Beutel, A.
\newblock Counterfactual fairness in text classification through robustness.
\newblock In \emph{Proceedings of the 2019 AAAI/ACM Conference on AI, Ethics,
  and Society}, AIES '19, pp.\  219–226, New York, NY, USA, 2019. Association
  for Computing Machinery.
\newblock ISBN 9781450363242.
\newblock \doi{10.1145/3306618.3317950}.
\newblock URL \url{https://doi.org/10.1145/3306618.3317950}.

\bibitem[Gibbs \& Su(2002)Gibbs and Su]{gibbs2002choosing}
Gibbs, A.~L. and Su, F.~E.
\newblock On choosing and bounding probability metrics.
\newblock \emph{International statistical review}, 70\penalty0 (3):\penalty0
  419--435, 2002.

\bibitem[Gitiaux \& Rangwala(2021)Gitiaux and Rangwala]{pmlr-v130-gitiaux21a}
Gitiaux, X. and Rangwala, H.
\newblock Learning smooth and fair representations.
\newblock In Banerjee, A. and Fukumizu, K. (eds.), \emph{Proceedings of The
  24th International Conference on Artificial Intelligence and Statistics},
  volume 130 of \emph{Proceedings of Machine Learning Research}, pp.\
  253--261. PMLR, 13--15 Apr 2021.
\newblock URL \url{https://proceedings.mlr.press/v130/gitiaux21a.html}.

\bibitem[Goodfellow et~al.(2014)Goodfellow, Pouget-Abadie, Mirza, Xu,
  Warde-Farley, Ozair, Courville, and Bengio]{NIPS2014_5ca3e9b1}
Goodfellow, I., Pouget-Abadie, J., Mirza, M., Xu, B., Warde-Farley, D., Ozair,
  S., Courville, A., and Bengio, Y.
\newblock Generative adversarial nets.
\newblock In Ghahramani, Z., Welling, M., Cortes, C., Lawrence, N., and
  Weinberger, K.~Q. (eds.), \emph{Advances in Neural Information Processing
  Systems}, volume~27. Curran Associates, Inc., 2014.
\newblock URL
  \url{https://proceedings.neurips.cc/paper/2014/file/5ca3e9b122f61f8f06494c97b1afccf3-Paper.pdf}.

\bibitem[Hardt et~al.(2016)Hardt, Price, and Srebro]{hardt2016equality}
Hardt, M., Price, E., and Srebro, N.
\newblock Equality of opportunity in supervised learning.
\newblock In \emph{Advances in neural information processing systems}, pp.\
  3315--3323, 2016.

\bibitem[Kantorovich \& Rubinstein(1958)Kantorovich and Rubinstein]{KR:58}
Kantorovich, L. and Rubinstein, G.~S.
\newblock On a space of totally additive functions.
\newblock \emph{Vestnik Leningrad. Univ}, 13:\penalty0 52--59, 1958.

\bibitem[Kleinberg et~al.(2018)Kleinberg, Ludwig, Mullainathan, and
  Rambachan]{kleinberg2018algorithmic}
Kleinberg, J., Ludwig, J., Mullainathan, S., and Rambachan, A.
\newblock Algorithmic fairness.
\newblock In \emph{Aea papers and proceedings}, volume 108, pp.\  22--27, 2018.

\bibitem[Kusner et~al.(2017)Kusner, Loftus, Russell, and
  Silva]{NIPS2017_a486cd07}
Kusner, M.~J., Loftus, J., Russell, C., and Silva, R.
\newblock Counterfactual fairness.
\newblock In Guyon, I., Luxburg, U.~V., Bengio, S., Wallach, H., Fergus, R.,
  Vishwanathan, S., and Garnett, R. (eds.), \emph{Advances in Neural
  Information Processing Systems}, volume~30. Curran Associates, Inc., 2017.
\newblock URL
  \url{https://proceedings.neurips.cc/paper/2017/file/a486cd07e4ac3d270571622f4f316ec5-Paper.pdf}.

\bibitem[Lohaus et~al.(2020)Lohaus, Perrot, and Luxburg]{pmlr-v119-lohaus20a}
Lohaus, M., Perrot, M., and Luxburg, U.~V.
\newblock Too relaxed to be fair.
\newblock In III, H.~D. and Singh, A. (eds.), \emph{Proceedings of the 37th
  International Conference on Machine Learning}, volume 119 of
  \emph{Proceedings of Machine Learning Research}, pp.\  6360--6369. PMLR,
  13--18 Jul 2020.
\newblock URL \url{https://proceedings.mlr.press/v119/lohaus20a.html}.

\bibitem[Louizos et~al.(2015)Louizos, Swersky, Li, Welling, and
  Zemel]{https://doi.org/10.48550/arxiv.1511.00830}
Louizos, C., Swersky, K., Li, Y., Welling, M., and Zemel, R.
\newblock The variational fair autoencoder, 2015.
\newblock URL \url{https://arxiv.org/abs/1511.00830}.

\bibitem[Madras et~al.(2018)Madras, Creager, Pitassi, and
  Zemel]{Madras2018LearningAF}
Madras, D., Creager, E., Pitassi, T., and Zemel, R.~S.
\newblock Learning adversarially fair and transferable representations.
\newblock In \emph{ICML}, 2018.

\bibitem[Man(2017)]{man2017computing}
Man, Y.~K.
\newblock On computing the vandermonde matrix inverse.
\newblock In \emph{Proceedings of the World Congress on Engineering}, volume~1,
  2017.

\bibitem[McCullagh(1994)]{mccullagh1994does}
McCullagh, P.
\newblock Does the moment-generating function characterize a distribution?
\newblock \emph{The American Statistician}, 48\penalty0 (3):\penalty0 208--208,
  1994.

\bibitem[Mehrabi et~al.(2019)Mehrabi, Morstatter, Saxena, Lerman, and
  Galstyan]{mehrabi2019survey}
Mehrabi, N., Morstatter, F., Saxena, N., Lerman, K., and Galstyan, A.
\newblock A survey on bias and fairness in machine learning.
\newblock \emph{arXiv preprint arXiv:1908.09635}, 2019.

\bibitem[Mukherjee et~al.(2020{\natexlab{a}})Mukherjee, Yurochkin, Banerjee,
  and Sun]{face}
Mukherjee, D., Yurochkin, M., Banerjee, M., and Sun, Y.
\newblock Two simple ways to learn individual fairness metrics from data.
\newblock In \emph{Proceedings of the 37th International Conference on Machine
  Learning}, pp.\  7097--7107, 2020{\natexlab{a}}.

\bibitem[Mukherjee et~al.(2020{\natexlab{b}})Mukherjee, Yurochkin, Banerjee,
  and Sun]{sensr}
Mukherjee, D., Yurochkin, M., Banerjee, M., and Sun, Y.
\newblock Two simple ways to learn individual fairness metrics from data.
\newblock In \emph{Proceedings of the 37th International Conference on Machine
  Learning}, pp.\  7097--7107, 2020{\natexlab{b}}.

\bibitem[Noh et~al.(2015)Noh, Hong, and Han]{noh2015learning}
Noh, H., Hong, S., and Han, B.
\newblock Learning deconvolution network for semantic segmentation.
\newblock In \emph{Proceedings of the IEEE international conference on computer
  vision}, pp.\  1520--1528, 2015.

\bibitem[Quadrianto et~al.(2019{\natexlab{a}})Quadrianto, Sharmanska, and
  Thomas]{Quadrianto_2019_CVPR}
Quadrianto, N., Sharmanska, V., and Thomas, O.
\newblock Discovering fair representations in the data domain.
\newblock In \emph{Proceedings of the IEEE/CVF Conference on Computer Vision
  and Pattern Recognition (CVPR)}, June 2019{\natexlab{a}}.

\bibitem[Quadrianto et~al.(2019{\natexlab{b}})Quadrianto, Sharmanska, and
  Thomas]{quadrianto2019discovering}
Quadrianto, N., Sharmanska, V., and Thomas, O.
\newblock Discovering fair representations in the data domain.
\newblock In \emph{Proceedings of the IEEE/CVF Conference on Computer Vision
  and Pattern Recognition}, pp.\  8227--8236, 2019{\natexlab{b}}.

\bibitem[Ruoss et~al.(2020)Ruoss, Balunovic, Fischer, and
  Vechev]{ruoss2020learning}
Ruoss, A., Balunovic, M., Fischer, M., and Vechev, M.
\newblock Learning certified individually fair representations.
\newblock In \emph{Advances in Neural Information Processing Systems 33}. 2020.

\bibitem[Sharifi-Malvajerdi et~al.(2019)Sharifi-Malvajerdi, Kearns, and
  Roth]{average}
Sharifi-Malvajerdi, S., Kearns, M., and Roth, A.
\newblock Average individual fairness: Algorithms, generalization and
  experiments.
\newblock In \emph{Advances in Neural Information Processing Systems},
  volume~32, 2019.
\newblock URL
  \url{https://proceedings.neurips.cc/paper/2019/file/0e1feae55e360ff05fef58199b3fa521-Paper.pdf}.

\bibitem[Steinwart \& Christmann(2008)Steinwart and
  Christmann]{steinwart2008support}
Steinwart, I. and Christmann, A.
\newblock \emph{Support vector machines}.
\newblock Springer Science \& Business Media, 2008.

\bibitem[Villani(2008)]{Villani2008OptimalTO}
Villani, C.
\newblock Optimal transport: Old and new.
\newblock 2008.

\bibitem[Wu et~al.(2019{\natexlab{a}})Wu, Zhang, and
  Wu]{10.1145/3308558.3313723}
Wu, Y., Zhang, L., and Wu, X.
\newblock On convexity and bounds of fairness-aware classification.
\newblock In \emph{The World Wide Web Conference}, WWW '19, pp.\  3356–3362,
  New York, NY, USA, 2019{\natexlab{a}}. Association for Computing Machinery.
\newblock ISBN 9781450366748.
\newblock \doi{10.1145/3308558.3313723}.
\newblock URL \url{https://doi.org/10.1145/3308558.3313723}.

\bibitem[Wu et~al.(2019{\natexlab{b}})Wu, Zhang, and Wu]{ijcai2019-199}
Wu, Y., Zhang, L., and Wu, X.
\newblock Counterfactual fairness: Unidentification, bound and algorithm.
\newblock In \emph{Proceedings of the Twenty-Eighth International Joint
  Conference on Artificial Intelligence, {IJCAI-19}}, pp.\  1438--1444.
  International Joint Conferences on Artificial Intelligence Organization, 7
  2019{\natexlab{b}}.
\newblock \doi{10.24963/ijcai.2019/199}.
\newblock URL \url{https://doi.org/10.24963/ijcai.2019/199}.

\bibitem[Xu et~al.(2018)Xu, Yuan, Zhang, and Wu]{8622525}
Xu, D., Yuan, S., Zhang, L., and Wu, X.
\newblock Fairgan: Fairness-aware generative adversarial networks.
\newblock In \emph{2018 IEEE International Conference on Big Data (Big Data)},
  pp.\  570--575, 2018.
\newblock \doi{10.1109/BigData.2018.8622525}.

\bibitem[Xu et~al.(2020)Xu, Cui, Kuang, Li, Zhou, Shen, and Cui]{2020alg}
Xu, R., Cui, P., Kuang, K., Li, B., Zhou, L., Shen, Z., and Cui, W.
\newblock Algorithmic decision making with conditional fairness.
\newblock \emph{Proceedings of the 26th ACM SIGKDD International Conference on
  Knowledge Discovery Data Mining}, Jul 2020.
\newblock \doi{10.1145/3394486.3403263}.
\newblock URL \url{http://dx.doi.org/10.1145/3394486.3403263}.

\bibitem[Yona \& Rothblum(2018)Yona and Rothblum]{pacf}
Yona, G. and Rothblum, G.
\newblock Probably approximately metric-fair learning.
\newblock In Dy, J. and Krause, A. (eds.), \emph{Proceedings of the 35th
  International Conference on Machine Learning}, volume~80 of \emph{Proceedings
  of Machine Learning Research}, pp.\  5680--5688, Stockholmsmässan, Stockholm
  Sweden, 10--15 Jul 2018. PMLR.

\bibitem[Yukich et~al.(1995)Yukich, Stinchcombe, and White]{yukich1995sup}
Yukich, J.~E., Stinchcombe, M.~B., and White, H.
\newblock Sup-norm approximation bounds for networks through probabilistic
  methods.
\newblock \emph{IEEE Transactions on Information Theory}, 41\penalty0
  (4):\penalty0 1021--1027, 1995.

\bibitem[Zafar et~al.(2017)Zafar, Valera, Rogriguez, and
  Gummadi]{zafar2017fairness}
Zafar, M.~B., Valera, I., Rogriguez, M.~G., and Gummadi, K.~P.
\newblock Fairness constraints: Mechanisms for fair classification.
\newblock In \emph{Artificial Intelligence and Statistics}, pp.\  962--970,
  2017.

\bibitem[Zeiler(2012)]{journals/corr/abs-1212-5701}
Zeiler, M.~D.
\newblock Adadelta: An adaptive learning rate method.
\newblock \emph{CoRR}, abs/1212.5701, 2012.
\newblock URL
  \url{http://dblp.uni-trier.de/db/journals/corr/corr1212.html#abs-1212-5701}.

\bibitem[Zemel et~al.(2013)Zemel, Wu, Swersky, Pitassi, and
  Dwork]{pmlr-v28-zemel13}
Zemel, R., Wu, Y., Swersky, K., Pitassi, T., and Dwork, C.
\newblock Learning fair representations.
\newblock In Dasgupta, S. and McAllester, D. (eds.), \emph{Proceedings of the
  30th International Conference on Machine Learning}, volume~28 of
  \emph{Proceedings of Machine Learning Research}, pp.\  325--333, Atlanta,
  Georgia, USA, 17--19 Jun 2013. PMLR.
\newblock URL \url{https://proceedings.mlr.press/v28/zemel13.html}.

\bibitem[Zeng et~al.(2021)Zeng, Islam, Keya, Foulds, Song, and
  Pan]{zeng2021fair}
Zeng, Z., Islam, R., Keya, K.~N., Foulds, J., Song, Y., and Pan, S.
\newblock Fair representation learning for heterogeneous information networks,
  2021.

\end{thebibliography}
